\documentclass[sigconf,nonacm]{acmart}

\usepackage{array}
\newcolumntype{P}[1]{>{\centering\arraybackslash}p{#1}}
\usepackage{graphicx}%
\usepackage{multirow}%
\usepackage{amsmath,amsfonts}%
\usepackage{amsthm}%
\usepackage{mathrsfs}%
\usepackage[title]{appendix}%
\usepackage{xcolor}%
\usepackage{textcomp}%
\usepackage{manyfoot}%
\usepackage{booktabs}%
\usepackage{listings}%
\usepackage{dblfloatfix}

\usepackage{tikz}
\usetikzlibrary{fit}
\usetikzlibrary{arrows.meta, positioning}
\usetikzlibrary{matrix,shapes.geometric,shapes.misc}
\usetikzlibrary{calc}

\usepackage{nicematrix}

\usepackage{algorithm}
\usepackage[noend]{algpseudocode}
\usepackage{rotating}

\usepackage[english]{babel}
\usepackage{url}
\usepackage{tablefootnote}
\usepackage{booktabs}
\usepackage{xr}
\usepackage{pdfpages}

\usepackage{enumitem}

\newcommand{\stickfigures}{\ensuremath{\textsc{stickfigures}}\xspace}
\newcommand{\optdigits}{\ensuremath{\textsc{optdigits}}\xspace}

\tikzset{%
  point/.style={circle, inner sep=2pt}, 
  other point/.style={fill=black, point}  
}

\makeatletter
\providecommand*{\cupdot}{%
  \mathbin{%
    \mathpalette\@cupdot{}%
  }%
}
\newcommand*{\@cupdot}[2]{%
  \ooalign{%
    $\m@th#1\cup$\cr
    \sbox0{$#1\cup$}%
    \dimen@=\ht0 %
    \sbox0{$\m@th#1\cdot$}%
    \advance\dimen@ by -\ht0 %
    \dimen@=.5\dimen@
    \hidewidth\raise\dimen@\box0\hidewidth
  }%
}

\providecommand*{\bigcupdot}{%
  \mathop{%
    \vphantom{\bigcup}%
    \mathpalette\@bigcupdot{}%
  }%
}
\newcommand*{\@bigcupdot}[2]{%
  \ooalign{%
    $\m@th#1\bigcup$\cr
    \sbox0{$#1\bigcup$}%
    \dimen@=\ht0 %
    \advance\dimen@ by -\dp0 %
    \sbox0{\scalebox{2}{$\m@th#1\cdot$}}%
    \advance\dimen@ by -\ht0 %
    \dimen@=.5\dimen@
    \hidewidth\raise\dimen@\box0\hidewidth
  }%
}
\makeatother

\usepackage[utf8]{inputenc} % allow utf-8 input
\usepackage[T1]{fontenc}    % use 8-bit T1 fonts
\usepackage{booktabs}       % professional-quality tables
\usepackage{amsfonts}       % blackboard math symbols
\usepackage{nicefrac}       % compact symbols for 1/2, etc.
\usepackage{microtype}      % microtypography
\usepackage{amsmath}
\usepackage{mathtools}
\usepackage{graphicx}
\usepackage{xspace}
\usepackage{float}
\usetikzlibrary{calc,positioning}
\usepackage{varwidth}
\usepackage{enumitem}
\usetikzlibrary{positioning}

\usepackage{pgfplots}
\pgfplotsset{compat=newest}
\usepgfplotslibrary{groupplots}
\usepgfplotslibrary{dateplot}

\makeatletter
\newcommand\footnoteref[1]{\protected@xdef\@thefnmark{\ref{#1}}\@footnotemark}
\makeatother

\DeclareMathOperator*{\argmin}{\arg\min}

\newcommand{\spara}[1]{\smallskip\noindent{\bf{#1}}}
\renewcommand{\spara}[1]{\noindent\textbf{#1}}

\newcommand{\np}{\ensuremath{\mathbf{NP}}\xspace}

\newcommand{\dataset}{\ensuremath{\mathcal{D}}\xspace}

\newcommand{\inputmat}{\ensuremath{\mathbf{D}}\xspace}

\newcommand{\nrows}{\ensuremath{n}\xspace}

\newcommand{\actualrank}{\ensuremath{r}\xspace}

\newcommand{\hadkmeans}{\textsc{Khatri-Rao}-\ensuremath{k}-\textsc{Means}\xspace}
\newcommand{\hadkmeanssum}{\textsc{Khatri-Rao}-\ensuremath{k}-\textsc{Means}-$+$\xspace}
\newcommand{\hadkmeansprod}{\textsc{Khatri-Rao}-\ensuremath{k}-\textsc{Means}-$\times$\xspace}
\newcommand{\kmeans}{\ensuremath{k}-\textsc{Means}\xspace}
\newcommand{\kmeansprob}{\ensuremath{k}-Means\xspace}
\newcommand{\kmeanspp}{\ensuremath{k\textsc{-Means++}}\xspace}

\newcommand{\krkmeans}{\textsc{Khatri-Rao}-\ensuremath{k}-\textsc{Means}\xspace}

\newcommand{\dkm}{\textsc{DKM}\xspace}
\newcommand{\krdkm}{\textsc{Khatri-Rao DKM}\xspace} 

\newcommand{\idec}{\textsc{IDEC}\xspace}
\newcommand{\kridec}{\textsc{Khatri-Rao IDEC}\xspace} 

\newcommand{\budget}{\ensuremath{b}\xspace}

\newcommand{\A}{\ensuremath{\mathbf{A}}\xspace}
\newcommand{\B}{\ensuremath{\mathbf{B}}\xspace}

\newcommand{\cluster}{\ensuremath{C}\xspace}

\newcommand{\semicentroidword}{protocentroid\xspace}

\newcommand{\semicentroidsword}{protocentroids\xspace}

\newcommand{\ourkmeans}{Khatri-Rao \kmeansprob}

\newcommand{\aggregator}{\ensuremath{\oplus}\xspace}

\newcommand{\bigO}{\ensuremath{\mathcal{O}}\xspace}

\newcommand{\ssetot}[1]
{\ensuremath{E(\inputmat)}\xspace}

\newcommand{\tol}{\ensuremath{\epsilon}\xspace}

\newcommand{\mnist}{\ensuremath{\textsc{MNIST}}\xspace}

\newcommand{\clusteringfunction}{\ensuremath{f_{\mathcal{C}}}\xspace}

\newcommand{\clusteringfunctionquality}{\ensuremath{Q_{\mathcal{C}}}\xspace}

\newcommand{\allparams}{\ensuremath{\Theta}\xspace}

\newcommand{\allparamscentroids}{\ensuremath{\Theta_{\boldsymbol{\mu}}}\xspace}

\newcommand{\allparamsautoencoder}{\ensuremath{\Theta_{\boldsymbol{\alpha}}}\xspace}

\newcommand{\kmeanscentroids}{\ensuremath{\{\boldsymbol{\mu}_1 , \dots , \boldsymbol{\mu}_k\}}\xspace}

\newcommand{\protocentroids
}{\ensuremath{\boldsymbol{\theta}}\xspace}

\newcommand{\protocentroid
}{\ensuremath{\boldsymbol{\theta}}}

\newcommand{\nsets
}{\ensuremath{p}\xspace}

\newcommand{\krcentroid
}{\ensuremath{\boldsymbol{\mu}_i}\xspace}

\newcommand{\nlayers}{\ensuremath{n_l}\xspace}

\newcommand{\autoencoderweightmatrixl}{
\mathbf{W}_l\xspace}

\newcommand{\blobs}{\textsc{Blobs}\xspace}
\newcommand{\cmu}{\textsc{CMU Faces}\xspace}
\newcommand{\double}{\textsc{Double MNIST}\xspace}
\newcommand{\har}{\textsc{HAR}\xspace}
\newcommand{\olivetti}{\textsc{Olivetti Faces}\xspace}
\newcommand{\soybean}{\textsc{Soybean Large}\xspace}
\newcommand{\symbols}{\textsc{Symbols}\xspace}
\newcommand{\chameleon}{\textsc{Chameleon}\xspace}
\newcommand{\classification}{\textsc{Classification}\xspace}

\newcommand{\rfifteen}{\textsc{R15}\xspace}

\newcommand{\ndatapoints}{\ensuremath{n}\xspace}
\newcommand{\xbf}{\ensuremath{\mathbf{x}}\xspace}
\newcommand{\nfeats}{\ensuremath{m}\xspace}
\newcommand{\centroid}{\ensuremath{\boldsymbol{\mu}}\xspace}
\newcommand{\nclusters}{\ensuremath{k}\xspace}
\newcommand{\protocentroidsetcardinality}{\ensuremath{h}\xspace}
\newcommand{\card}{\protocentroidsetcardinality}

\newcommand{\centroidsset}{\ensuremath{M_{\boldsymbol{\mu}}}\xspace}

\newcommand{\zerovector}{\ensuremath{\mathbf{0}}\xspace}

\newcommand{\abf}{\ensuremath{\mathbf{a}}\xspace}
\newcommand{\krdcfull}{Khatri-Rao deep clustering\xspace}

\newcommand{\FEMNIST}{\textsc{FEMNIST}\xspace}

\newcommand{\fedkmeans}{\textsc{F}\ensuremath{k}\textsc{M}\xspace}
\newcommand{\fedhadkmeans}{\textsc{Khatri-Rao}-\textsc{F}\ensuremath{k}\textsc{M}\xspace}

\newcommand{\kr}{Khatri-Rao\xspace}

\newcommand{\naive}{na\"ive
\xspace}

\newcommand{\amat}{\ensuremath{\mathbf{A}}\xspace}

\newcommand{\bmat}{\ensuremath{\mathbf{B}}\xspace}

\newcommand{\nhadamardfactors}{\ensuremath{q}\xspace}

\newtheorem{problem}{Problem}
\renewcommand\footnotetextcopyrightpermission[1]{} % no permission footnote
\setcopyright{none}                      % no ACM copyright text

\begin{document}

%% EDBT rule: --> Please use ``titlecase'' in the title!

\title{Khatri-Rao Clustering for Data Summarization}

%% EDBT rule: --> At least 1 (the corresponding) author needs to have a
%%                registered ORCID, ideally all authors have one

\author{Martino Ciaperoni}
\affiliation{%
  \institution{Scuola Normale Superiore}
  \city{Pisa}
  \country{Italy}
}
\email{martino.ciaperoni@sns.it}
\orcid{0009-0009-7581-2031}

\author{Collin Leiber}
\affiliation{%
  \institution{Aalto University}
  \city{Espoo}
  \country{Finland}
}
\email{collin.leiber@aalto.fi}
\orcid{0000-0001-5368-5697}

\author{Aristides Gionis}
\affiliation{%
  \institution{KTH Royal Institute of Technology}
  \institution{Digital Futures}
  \city{Stockholm}
  \country{Sweden}
}
\email{argioni@kth.se}
\orcid{0000-0002-5211-112X}

\author{Heikki Mannila}
\affiliation{%
  \institution{Aalto University}
  \city{Espoo}
  \country{Finland}
}
\email{heikki.mannila@aalto.fi}
\orcid{0009-0000-3772-6073}


\renewcommand{\shortauthors}{Ciaperoni et al.} %%MS: removed contents
%\renewcommand{\shorttitle}{} %%MS: added this one

\begin{abstract}
  As datasets continue to grow in size and complexity, finding succinct yet accurate data summaries poses a key challenge. Centroid-based clustering, a widely adopted approach to address this challenge, finds informative summaries of datasets in terms of few prototypes, each representing a cluster in the data. Despite their wide adoption, the resulting data summaries often contain redundancies, limiting their effectiveness particularly in datasets characterized by a large number of underlying clusters. To overcome this limitation, we introduce the \kr clustering paradigm that extends traditional centroid-based clustering to produce more succinct but equally accurate data summaries by postulating that centroids arise from the interaction of two or more succinct sets of protocentroids. 

  We study two central approaches to centroid-based clustering, namely the well-established \kmeans algorithm and the increasingly popular topic of deep clustering, under the lens of the \kr paradigm. To this end, we introduce the \hadkmeans algorithm and the \krdcfull framework. Extensive experiments show that \hadkmeans can strike a more favorable trade-off between succinctness and accuracy in data summarization than standard \kmeans. Leveraging representation learning, the \krdcfull framework offers even greater benefits, reducing even more the size of data summaries given by deep clustering while preserving their accuracy.
\end{abstract}

\keywords{Data summarization, centroid-based clustering, deep clustering}

\maketitle

\section{Introduction}\label{sec:introduction}

Distilling datasets into succinct and meaningful summaries represents an essential and ubiquitous task. Clustering has emerged as a core methodology for tackling this problem. In particular, centroid-based clustering holds a central position when looking for effective data-summarization strategies. Given a notion of similarity, centroid-based clustering algorithms summarize data by assigning each data point to the closest element in a succinct set of representative data points called \emph{centroids}. The data summaries provided by centroid-based clustering identify meaningful clusters and reveal patterns useful for applications such as exploratory data analysis, segmentation, anomaly detection, and document organization~\cite{xu2015comprehensive}.

\begin{figure}[t!]
\centering
\includegraphics[width=0.29\textwidth]{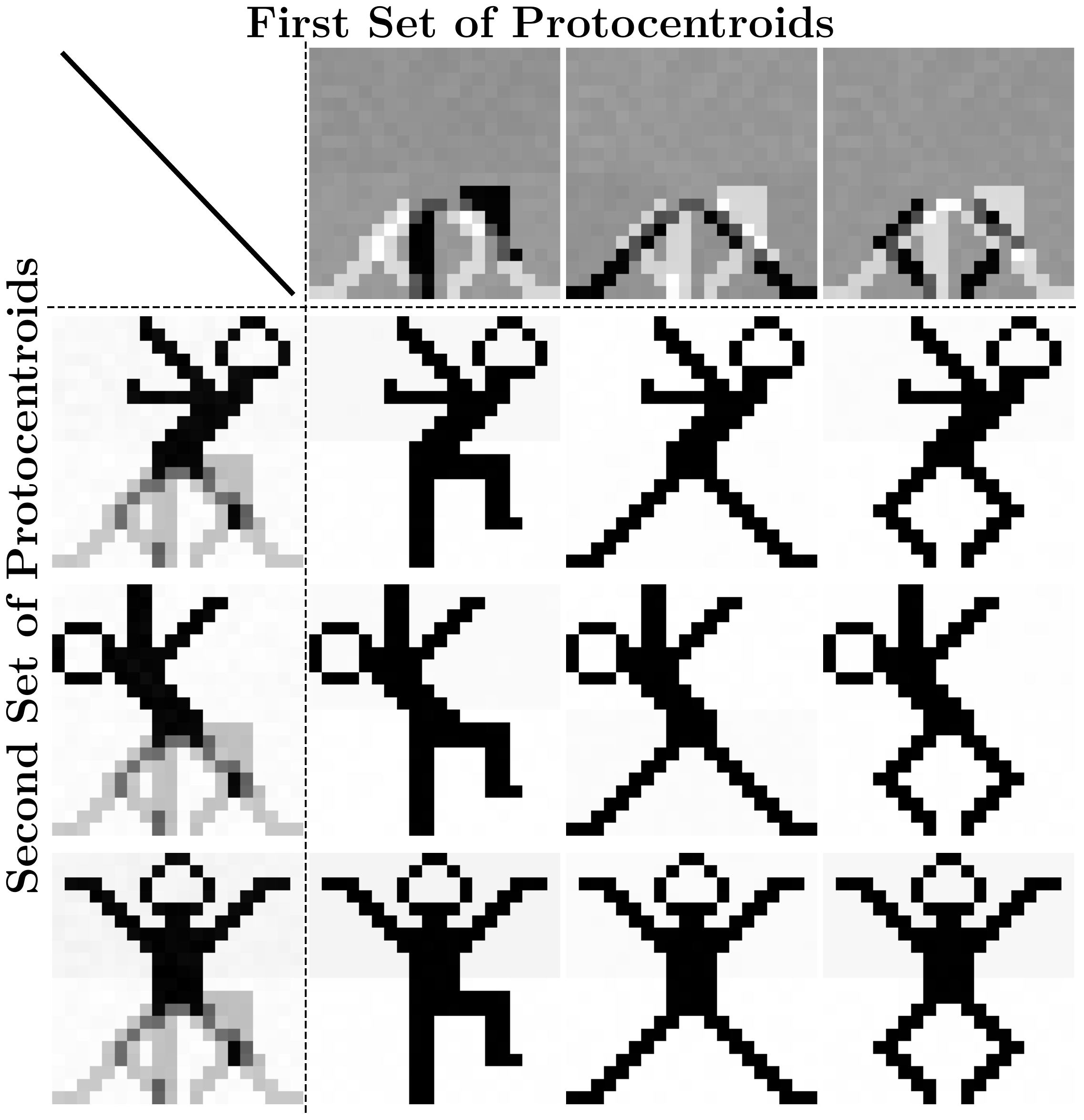}
\caption{\stickfigures dataset. Example of two sets of $3$  \semicentroidsword interacting additively to generate $9$ centroids.}
    \label{fig:stickfigures_example}
\end{figure}

In many cases, the interest lies in extracting a data summary which is as succinct and accurate as possible. Clustering has consistently shown success in summarizing data across diverse domains, for example, in image~\cite{omari2025advancing}, document~\cite{dvorsky2010document}, tabular~\cite{han2016lossless}, and spatio-temporal~\cite{whelan2010data} data processing, as well as for compressing deep neural networks~\cite{son2018clustering,somasundaram2011novel}. 

\spara{Pushing the boundaries of centroid-based clustering for data summarization through Khatri-Rao operators.}
Many modern applications demand the summarization of datasets of unprecedented scale~\cite{draganov2024settling}. As the size of a dataset increases, the underlying cluster structure may evolve either by enlarging clusters or by introducing additional clusters. Thus, modern datasets can exhibit a massive number of underlying clusters~\cite{kobren2017hierarchical,bateni2021extreme}. For instance, studies on protein structures analyze millions of clusters~\cite{barrio2023clustering}. Similarly, a large number of clusters can be incurred in topic modeling for documents~\cite{sethia2022framework}. In entity resolution, including record linkage and de-duplication, the number of clusters grows with dataset size~\cite{betancourt2016flexible}, and in real-world networks the number of clusters tends to follow a power-law distribution that also increases with network size~\cite{clauset2004finding}. Centroid-based clustering struggles to yield data summaries that are both accurate and succinct in similar large-scale applications characterized by a massive number of clusters. 

In spite of this, no existing work challenges the long-standing centroid-based clustering paradigm toward more succinct data summarization. To address this gap, we adopt a data-compression perspective on centroid-based clustering, and we investigate the following research question: \textit{do standard centroid-based clustering algorithms produce data summaries that carry some degree of redundancy, suggesting the potential for further compression?} To shed light on this question, we relax a crucial assumption of centroid-based clustering, namely that centroids are independent entities. Often, richer representations emerge from the interaction of simpler building blocks, as illustrated, e.g., by multi-head attention~\cite{vaswani2017attention}. In this spirit, we postulate that centroids admit a more succinct representation in terms of a smaller set of building blocks. More specifically, we introduce the \kr clustering paradigm, in which centroids arise from the interaction of succinct sets of \emph{\semicentroidsword} through so-called \kr operators. These operators aggregate each \semicentroidword in each set with all \semicentroidsword of the other sets via element\-wise operations such as sums or products. 

A simple example of this formulation is given in Figure~\ref{fig:stickfigures_example}. It shows that the nine clusters in the \stickfigures dataset~\cite{gunnemann2014smvc} can be represented by two sets of three \semicentroidsword. Three \semicentroidsword explain the upper part of the stick figures and three explain the lower part. The final centroids can then be generated by additively combining each \semicentroidword in one set with every \semicentroidword in the other set. 
In this case, centroids are said to exhibit a Khatri-Rao structure, and more precisely they correspond to the Khatri-Rao sum of the two sets of \semicentroidsword. 
The stick-figure example shows that the dataset can be summarized by just $6$ images, whereas standard centroid-based clustering algorithms require $9$ images.
Extending beyond the example, $\nsets$ sets of  $\card_1, \card_2, \dots$ and  $\card_\nsets$ \semicentroidsword can represent up to $\prod_{i=1}^\nsets \card_i$ centroids. 

A naïve approach to obtain such a succinct clustering-based description of a dataset would be to start from a set of centroids and subsequently extract an approximation of the centroids satisfying the Khatri-Rao structure. 
Nevertheless, for any dataset, multiple centroid-based summaries may achieve nearly the same succinctness and accuracy, some of which may (approximately or exactly) admit an even more succinct representation using Khatri–Rao operators, while others may not. Thus, the Khatri-Rao clustering paradigm we introduce does not start by finding centroids through standard approaches, but extends standard approaches to directly target data summaries satisfying the Khatri-Rao structure. 

As a first concrete instantiation of the Khatri-Rao clustering paradigm, we design \hadkmeans, which builds on the most popular algorithm for centroid-based clustering, i.e., the \kmeans algorithm~\cite{lloyd1982least}, to find centroids that can be succinctly expressed as the \kr sum or product of two or more sets of \semicentroidsword. 
Although \hadkmeans can result in considerably more succinct and yet equally accurate data summaries than standard \kmeans, it is more prone to converge to undesirable local minima. Finding high-quality solutions may require exploring many different initializations. Motivated by this limitation, we introduce the \kr deep clustering framework, which starts from \hadkmeans and extends the \kr paradigm to deep clustering. Deep clustering effectively handles large-scale and high-dimensional data by incorporating deep-learning-based representation learning. By aligning the learned representations with the Khatri–Rao structure, the \kr deep clustering framework overcomes the limitations of \hadkmeans, and usually provides data summaries that have a comparable accuracy to unconstrained baselines but are consistently more succinct. 

Figure~\ref{fig:tradeoff} anticipates the benefits of Khatri-Rao clustering by comparing \kmeans and the deep clustering algorithms \textsc{Improved Deep Embedded Clustering} (\idec)~\cite{Guo2017ImprovedDE} and \textsc{Deep-k-Means} (\dkm)~\cite{fard2020deep} against their extensions based on the  Khatri-Rao paradigm. The \kr clustering algorithms maintain comparable accuracy (here measured by \textit{unsupervised clustering accuracy}~\cite{Yang2010accuracy}) and simultaneously achieve a high level of compression.

\begin{figure}[t!]
%\centering
%\scalebox{0.7}{\input{figures/firstPlot.tikz}}
\includegraphics[width=0.36\textwidth]{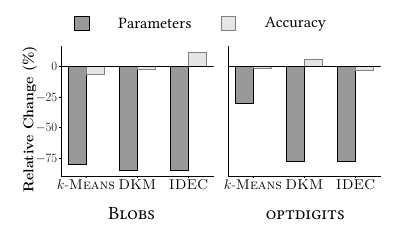}
%\makebox[0.4\textwidth]{%
%  \makebox[0.5\linewidth][c]{\hspace*{1.8cm}\blobs}%
%  \makebox[0.5\linewidth][c]{\hspace*{-0.5cm}\optdigits}%
%}
\caption{For a synthetic (\blobs) and a real (\optdigits) dataset, we show relative percentage changes in unsupervised clustering accuracy and parameter count for clustering solutions from algorithms based on the \kr paradigm relative to the baseline algorithms \kmeans, \textsc{Deep-k-Means} (\dkm) and \textsc{Improved Deep Embedded Clustering} (\idec).}
\label{fig:tradeoff}
\end{figure}

\spara{Contributions.}
Our contributions can be summarized as follows. 
\begin{itemize}
    \item We formalize the Khatri-Rao clustering paradigm, outlining its general principles and formulating the \ourkmeans and \kr deep clustering problems as concrete instantiations of the paradigm. 
    \item We introduce the \hadkmeans clustering algorithm and the Khatri-Rao deep clustering framework to effectively address the formulated problems. 
    \item Through extensive experiments, we show that \krkmeans can improve the trade-off between data summary size and accuracy over standard \kmeans. Even more remarkably, the Khatri–Rao deep clustering framework can compress the data summaries produced by deep clustering algorithms by up to $85\%$ with negligible loss in accuracy.
\end{itemize}

\spara{Roadmap.}
The rest of the paper is organized as follows. Section~\ref{sec:related} discusses related work. Section~\ref{sec:preliminaries} introduces the background. Section~\ref{sec:problem_formulation} formalizes the \kr clustering paradigm, and introduces the \ourkmeans and \kr deep clustering problem statements. Section~\ref{sec:naive} discusses a \naive approach to \kr clustering. Sections~\ref{sec:algorithms_shallow} and~\ref{sec:deep_clustering} present the \hadkmeans algorithm and the Khatri-Rao deep clustering framework, respectively. Section \ref{sec:designchoices} guides design choices in \kr clustering, and Section~\ref{sec:experiments} contains the experimental evaluation. Finally, conclusions are drawn in Section~\ref{sec:conclusions}. An appendix and our implementation are available online\footnote{\url{https://github.com/maciap/KhatriRaoClustering}}.

\section{Related Work}\label{sec:related}

In this section, we review prior work relevant to the present study.

\spara{Clustering, data summarization and compression.}
Obtaining data summaries that are as succinct and as accurate as possible has established itself as a central problem in such domains as database systems under various names. It is useful to distinguish between data summarization and compression. The goal of data summarization is to craft a summary of a dataset that captures its essential patterns. Centroid-based clustering is a popular approach for data summarization, although alternative approaches exist (e.g., aggregation, dimensionality reduction, or sampling) and may better suit particular applications~\cite{ahmed2019data}. A related field to data summarization is data compression, which emphasizes reducing data size for storage or transmission, either with or without information loss, over capturing the main patterns in the data~\cite{jayasankar2021survey}. The \kr clustering paradigm effectively compresses the data summaries given by centroid-based clustering. Thus, our work pertains to both data summarization and compression. We henceforth use “compression” to refer specifically to the reduction of the size of a data summary. 

Research has leveraged information-theoretic principles to establish a theoretical foundation for understanding clustering as a compression problem, explicitly considering the trade-off between compression, efficiency and clustering quality~\cite{bischof1999mdl,slonim2005information,raj2009information}. However, there is a lack of practical methods designed to enhance the compression capabilities of clustering, a gap we aim to address in this work. We concentrate on two central approaches to centroid-based clustering, \kmeans and deep clustering.

\spara{\kmeansprob clustering.}
\kmeansprob is a seminal problem in clustering~\cite{lloyd1982least}, 
and the classic \kmeans algorithm has established itself as the cornerstone of many clustering algorithms~\cite{jain2008data}. More details on the \kmeansprob problem and on the classic \kmeans algorithm are given in Section~\ref{sec:preliminaries}. Today, \kmeansprob clustering is still an active topic of research~\cite{ikotun2023k}. Nonetheless, prior to our work which extends \kmeans to the \kr paradigm, improving the compression power of \kmeansprob has been largely overlooked.

\spara{Deep clustering.}
Unlike \kmeans, which performs clustering directly in the input space, recent deep clustering algorithms combine clustering with neural networks to learn representations more amenable to clustering~\cite{leiber2023benchmarking,ren2024deep,zhou2025comprehensive}.

Deep clustering algorithms can use various types of neural networks. We focus on autoencoder-based algorithms as they are agnostic to data type and easily extendable~\cite{leiber2025introductory}. Centroid-based deep clustering algorithms are useful for summarization purposes and have demonstrated state-of-the-art performance \cite{miklautz2025breaking}. A well-known representative is \textsc{Deep-k-Means} (\dkm)~\cite{fard2020deep}. \dkm is similar to \kmeans. However, it introduces soft cluster assignments based on the softmax function. \textsc{Deep Embedded Clustering} (DEC)~\cite{xie2016unsupervised} and its successor \textsc{Improved Deep Embedded Clustering} (\idec)~\cite{Guo2017ImprovedDE} utilize the Kullback-Leibler divergence~\cite{kullback1951information} to learn an embedding that aligns a \emph{Student's t} model of the data distribution with a target distribution. Our work extends the \dkm and \idec algorithms to the \kr paradigm. Further details on deep clustering, \dkm and \idec are given in Section~\ref{sec:preliminaries}. 

\spara{Matrix decomposition.}
Clustering is closely related to matrix decomposition, which expresses a matrix as a product of two simpler matrices. It is known that, given a data matrix, the model imposed by a centroid-based clustering algorithm like \kmeans can be seen as a special case of a matrix-decomposition model where one of the factor matrices is constrained to indicate cluster assignments and the other stores centroids~\cite{ding2005equivalence}. Similarly, \kr clustering hinges on known operators like Hadamard (i.e., element\-wise) and \kr products~\cite{khatri1968solutions} and is closely connected to the Hadamard decomposition, which models a data matrix as the Hadamard product of matrix decompositions~\cite{ciaperoni2024hadamard}. In particular, the model imposed by \kr clustering with \semicentroidsword aggregated via element\-wise product on a data matrix is equivalent to a Hadamard decomposition where one matrix in each decomposition is constrained to indicate the assignment to \semicentroidsword. In general, algorithms for matrix decomposition can be useful for clustering~\cite{li2013nonnegative} and can also be potentially extended to the \kr clustering setting.

\section{Preliminaries}\label{sec:preliminaries}

Before we introduce the main ideas behind the proposed methods, we present the notation and preliminary notions that are used throughout the paper. Additionally, we review the essential clustering background required to understand our contributions.

\spara{Notation and basic definitions.}
Vectors are denoted by lower-case bold letters, and matrices by upper-case bold letters. Similarly, sets are denoted by upper-case letters and scalars by lower-case letters. Greek letters are reserved for data-summary parameters. The product of two scalars $a$ and $b$ is denoted by $ab$, while $\A \B$ and $\A \odot \B$ denote the standard and Hadamard matrix products, respectively. The Euclidean norm of $\mathbf{x}$ is denoted by $\| \mathbf{x} \|$. 

We consider datasets $\dataset \subset \mathbb{R}^{\nfeats}$ composed of a set of data points $\{\xbf_1, \xbf_2, \ldots, \xbf_\ndatapoints\}$, where $\xbf_i = (x_1, x_2 \dots, x_\nfeats)$. The goal of clustering is to divide the dataset into $\nclusters$ clusters \( \{\cluster_1, \cluster_2, \ldots, \cluster_k\} \) so that $\dataset=\cup_{i=1}^kC_i$ and $C_i \cap_{i \neq j} C_j = \emptyset$. Furthermore, we denote the set of cluster centroids by $\centroidsset = \{  \centroid_1, \centroid_2, \dots, \centroid_\nclusters \} \subset \mathbb{R}^{\nfeats}$. In the \kr clustering paradigm, we define the centroids in $\centroidsset$ more succinctly by combining \semicentroidsword. Unless stated otherwise, we assume that there are \nclusters centroids and \nsets sets of \semicentroidsword, where the $i$-th set has cardinality $\protocentroidsetcardinality_i$. The \semicentroidsword in different sets are combined using a function referred to as \emph{aggregator} and denoted by \aggregator. While the \kr  clustering paradigm is general and could in principle accommodate any aggregator function, in this work, we focus on the sum and product, i.e., $\aggregator \in \{  + , \times \}$.  When \aggregator is left unspecified, it is understood to represent an arbitrary choice of aggregator. For notational convenience, \aggregator is applied to scalars, vectors, matrices and sets. When applied to vectors and matrices, it is an element\-wise operator, corresponding to the standard sum for $\aggregator = +$ and to the Hadamard product for $\aggregator = \times$. When applied to sets (e.g., sets of \semicentroidsword), it is a so-called \kr  operator. We define the \kr \aggregator operator as an operator that, given \nsets sets of vectors, produces the set of all vectors obtained by element\-wise application of \aggregator to all possible combinations with one vector from each set. The name “Khatri-Rao” operator is chosen since, if \semicentroidsword in each set are arranged as rows of matrices and $\aggregator = \times$, the operator reduces to the \kr product~\cite{khatri1968solutions}. 

\spara{Clustering for data summarization.}
From the perspective of this work, clustering algorithms are functions \(\clusteringfunction: \dataset \rightarrow \allparams
\)
mapping a dataset \dataset to a succinct representation $\allparams$. The function \clusteringfunction is chosen to minimize the size of $\allparams$ while optimizing a measure of quality $\clusteringfunctionquality(\dataset,\allparams)$. All clustering algorithms discussed in this work follow from particular constraints imposed on $\allparams$ as well as different choices of $\clusteringfunctionquality(\dataset,\allparams)$. 

While the \kr paradigm can be applied to any centroid-based clustering algorithm, we focus on two popular approaches to clustering: the seminal \kmeansprob clustering and the emerging centroid-based deep clustering. In the remainder of this section, we review the basic principles underlying both approaches. 

\spara{\kmeans clustering.}
In \kmeansprob, the dataset is summarized in terms of a set of centroids $\allparams = \{ \boldsymbol{\mu}_i\}_{i=1}^\nclusters$, and the measure of quality is the total squared Euclidean distance of each point to the closest centroid, henceforth called \emph{inertia}, as is standard in the literature~\cite{scikit-learn}. 

Let $\centroid_i$ be the centroid of cluster $\cluster_i$. The \kmeansprob clustering problem asks to partition a dataset \( \dataset = \{\xbf_1, \xbf_2, \ldots, \xbf_\ndatapoints \}\) in \( \mathbb{R}^\nfeats \)  
into \( \nclusters \) clusters \( \{\cluster_1, \cluster_2, \ldots, \cluster_k\} \) such that the inertia
\begin{equation}\label{eq:kmeansObjective}
    \mathcal{Q}_C(\dataset, \allparams) = \sum_{i=1}^{\nclusters} \sum_{\xbf \in \cluster_i} \| \xbf - \centroid_i \|^2
\end{equation}
is minimized. 

To address this problem, the standard \kmeans algorithm starts by initializing \nclusters cluster centroids. Centroids can be initialized by sampling data points uniformly at random. Alternatively, the popular $\kmeanspp$~\cite{arthur2007k} strategy chooses data points that are \emph{sufficiently} far apart from each other as initial centroids, which gives theoretical approximation guarantees and often results in performance improvements. After initialization, the \kmeans algorithm iteratively assigns each data point to its nearest (in Euclidean distance) centroid and updates the cluster centroids by computing the mean of the points assigned to each cluster. The algorithm terminates and outputs the current centroids when either the centroids converge to a stable position or a maximum number of iterations is reached. 

\spara{Deep clustering.}
Deep clustering combines clustering with deep neural networks to perform some kind of representation learning. Unlike more traditional methods that rely on fixed features, it jointly optimizes feature extraction and cluster assignment in an end-to-end manner. In this study, we concentrate on autoencoder-based, centroid-based deep clustering approaches. In this setting, the input dataset \dataset is summarized as $\allparams = \allparamsautoencoder \cup \allparamscentroids$, where $\allparamsautoencoder$ and $\allparamscentroids$ denote the autoencoder and centroid parameters, respectively, and the quality function \clusteringfunctionquality to optimize captures the trade-off between clustering quality in the latent space and quality of reconstruction of \dataset from its latent-space representation. 

Formally, given a dataset $\dataset = \{\xbf_1, \xbf_2, \dots, \xbf_n\}$, let $f^e_{\allparamsautoencoder} : \dataset \to Z$ be a parametric mapping (encoder) to latent representations $Z = \{{\bf z}_1, \dots, {\bf z}_n\} \subset \mathbb{R}^{m'}$ and $f^d_{\allparamsautoencoder} : Z \to \hat{\dataset}$ a mapping (decoder) back to the original feature space. Further, let $\mathcal{L}_{\mathrm{rec}}(\dataset, f^d_{\allparamsautoencoder}(Z))$ denote a reconstruction loss and $\mathcal{L}_{\mathrm{cluster}}$ a loss capturing clustering quality in the latent space. Then, the deep clustering problem can be framed as the problem of minimizing
\begin{align}
\label{eq:deepclusteringObjective}
\mathcal{Q}_{C}(\dataset, \allparams) = \mathcal{L}_{\mathrm{cluster}}(Z, \allparamscentroids) + w_{rec}  \mathcal{L}_{\mathrm{rec}}(\dataset, f^d_{\allparamsautoencoder}(Z)),
\end{align}
where  $\allparamscentroids = \{\boldsymbol{\mu}_1, \dots, \boldsymbol{\mu}_\nclusters\}$ denotes the cluster centroids in the latent space and $w_{rec} \ge 0$ balances clustering and representation learning. In our study, we consider two clustering loss functions for $\mathcal{L}_{\mathrm{cluster}}$; the ones proposed for the \dkm~\cite{fard2020deep} and \idec~\cite{Guo2017ImprovedDE} algorithms. Thus,  $\mathcal{L}_{\mathrm{cluster}} \in \{\mathcal{L}_{\mathrm{DKM}}, \mathcal{L}_{\mathrm{IDEC}}\}$, where
\begin{align}
    \label{eq:dkm_loss}
    \mathcal{L}_{\mathrm{DKM}}(Z, \allparamscentroids) = \frac{1}{\ndatapoints} \sum_{z\in Z}\sum_{i=1}^{\nclusters} {||{\bf z} - \centroid_i||}^2 \frac{e^{-a{||{\bf z} - \centroid_i||}^2}}{\sum_{j=1}^{\nclusters} e^{-a{||{\bf z} - \centroid_j||}^2}}
\end{align}
and
\begin{align}
    \label{eq:idec_loss}
    \mathcal{L}_{\mathrm{IDEC}}(Z, \allparamscentroids) = \frac{1}{\ndatapoints} \sum_{l=1}^{\ndatapoints} \sum_{i=1}^{\nclusters} p_{l,i} \log(\frac{p_{l,i}}{q_{l,i}}).
\end{align}
Here, 
\begin{align*}
    q_{l,i}=\frac{(1+{||{\bf z}_l-\centroid_i||}_2^2)^{-\frac{a+1}{2}}}{\sum_{j=1}^{\nclusters}(1+{||{\bf z}_l-\centroid_j||}_2^2)^{-\frac{a+1}{2}}}~\text{and}~  p_{l,i}=\frac{q_{l,i}^2/\sum_{t=1}^{n} q_{t,i}}{\sum_{j=1}^{\nclusters}(q_{l,j}^2/\sum_{t=1}^{n} q_{t,j})},
\end{align*}
with $q_{l,i}$ and  $p_{l,i}$ representing the data and target distributions, respectively, and $a$ is an algorithm-specific parameter that is usually set to $1$ for \idec and $1000$ for \dkm. To obtain a succinct data summary in terms of autoencoder and centroid parameters, the objective in Equation~\eqref{eq:deepclusteringObjective} is typically optimized via batch-wise backpropagation, using automatic differentiation~\cite{leiber2023benchmarking}.

\section{Problem Formulation}\label{sec:problem_formulation}

Given any clustering algorithm \(
\clusteringfunction: \dataset \rightarrow \allparams,
\) we introduce its \kr extension by adopting its quality function $\clusteringfunctionquality(\dataset,\allparams)$ and introducing particular constraints on the parameters $\allparams$ that define the summary of \dataset associated with  \clusteringfunction. 

As the \kr clustering paradigm is designed for centroid-based clustering, all  algorithms that can be framed within the \kr clustering framework use a set of parameters to represent centroids. Thus, in all cases, given a user-specified integer $\nsets$, \kr clustering assumes that each centroid satisfies
\begin{equation}
\label{eq:centroidconstraint}
\krcentroid = \protocentroid^{j_1}_1 \aggregator \protocentroid^{j_2}_2 \dots  \aggregator \protocentroid^{j_\nsets}_\nsets \ \quad \forall i \in [1, \nclusters],
\end{equation}
where $\protocentroid^{j_l}_q$ denotes the $j_l$-th protocentroid vector in the $q$-th set of protocentroids. Accordingly, each centroid is associated with a cluster $\cluster_i$ and is uniquely identified by a tuple of $\nsets$ indices, one for each set of protocentroids, so that $\cluster_i = \cluster_{j_1,j_2, \dots  j_p}$, and similarly  $\centroid_i = \centroid_{j_1,j_2, \dots  j_p}$. Figure~\ref{fig:interactiondiagram} provides a schematic visualization describing how sets of protocentroids are combined to create cluster centroids. Additionally, Figure~\ref{fig:example_centroids} shows examples in synthetic data\footnote{Bottom dataset available at \url{https://github.com/milaan9/Clustering-Datasets}}. 

All algorithms based on the \kr clustering paradigm yield  succinct representations of centroid parameters by enforcing the constraint from Equation~\eqref{eq:centroidconstraint}. 
In some cases, as in the deep clustering setting, there are additional parameters (the autoencoder parameters) that \kr clustering seeks to reduce. 

\begin{figure}
    \centering
%\scalebox{0.8}{\input{figures/interaction_diagrams.tikz}}
	  \includegraphics[width=0.27\textwidth]{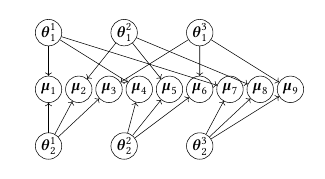}
    \caption{Diagram showing the interactions of two sets of \semicentroidsword to generate cluster centroids.}
    \label{fig:interactiondiagram}
\end{figure}

\begin{figure}[t!]
    \centering
    \begin{tabular}{c|c}
     \includegraphics[width=0.18\textwidth]{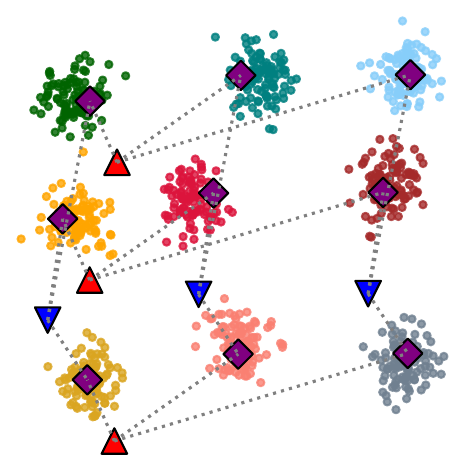}  & \includegraphics[width=0.18\textwidth]{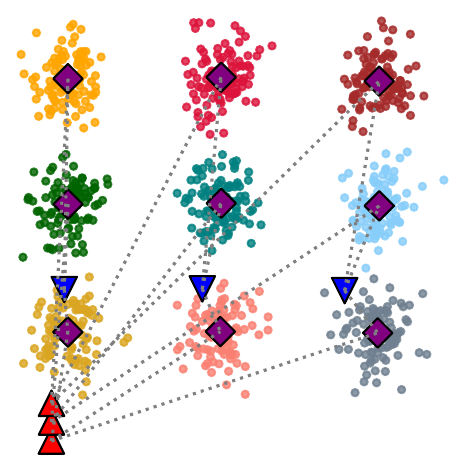} \\
     $\aggregator = +$ & $\aggregator = \times$ \\
    \midrule
\includegraphics[width=0.18\textwidth]{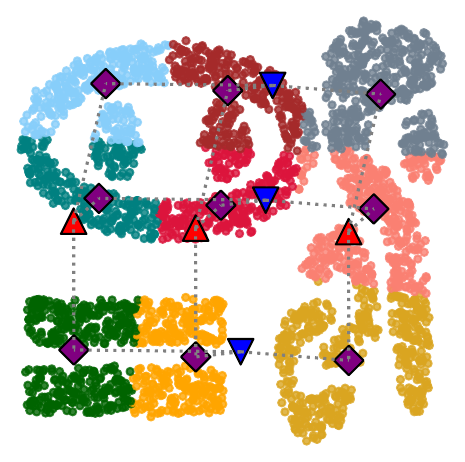}  & \includegraphics[width=0.18\textwidth]{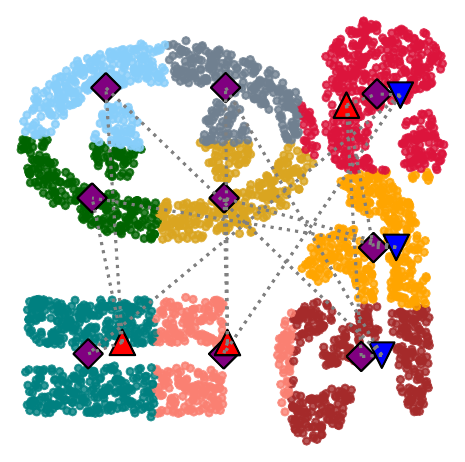} \\
     $\aggregator = +$ & $\aggregator = \times$ \\   
    \end{tabular}
    \caption{\kr-based (top) and arbitrarily-structured (bottom) synthetic data. Combining additively ($\aggregator=+$) or multiplicatively ($\aggregator=\times$) the first (red triangles) and second (blue triangles)  sets of \semicentroidsword yields the cluster centroids (purple diamonds). Gray lines indicate which \semicentroidword affects each cluster centroid.}
\label{fig:example_centroids}
\end{figure}

Although \kr clustering defines a general paradigm for centroid-based clustering, we focus on \kmeans and two deep clustering algorithms. \kr extensions of other centroid-based clustering algorithms, e.g., based on matrix decomposition~\cite{li2013nonnegative} or gradient descent~\cite{DBLP:conf/www/Sculley10}, are possible but require method-specific adjustments. Investigating whether our ideas can also be applied to clustering algorithms that do not rely on centroids, such as hierarchical clustering~\cite{ghosal2020short}, is left to future research. 

\subsection{\kr \kmeansprob Clustering}\label{sec:hadkmeansproblem}
In \kmeansprob clustering (introduced in Section~\ref{sec:preliminaries}), the dataset is summarized using a set of prototype vectors, i.e., the centroids $\allparams=\kmeanscentroids$, and the objective function \clusteringfunctionquality is the inertia in Equation~\eqref{eq:kmeansObjective}. 
%which quantifies the variance within each cluster. 

\kr \kmeansprob also seeks to minimize the inertia. However, it restricts all centroids $\allparams$ to arise from \kr aggregations of \semicentroidsword, so that $\allparams =\{  \protocentroid^{j_i}_l \mid j_i \in [1, \card_l]~\text{ and }~l \in [1,p] \}$. Formally, the \ourkmeans problem is defined as follows. 

\begin{problem}[\ourkmeans]\label{prob:kr_kmeans_problem}
    Given a dataset and an input $\nsets \in \mathbb{N}_+$, partition $\dataset$ into 
    $\prod_{l=1}^{\nsets} \card_l$
    %$\card_1 \card_2 \dots 
    %  \card_\nsets$ 
    clusters so as to minimize
\begin{equation*}\label{eq:khatri_rao_kmeans_objective}
\mathcal{Q}_{C}(\dataset, \allparams) = 
\sum_{j_1=1}^{\card_1}\sum_{j_2=1}^{\card_2}
    \dots
    \sum_{j_\nsets=1}^{\card_\nsets}
    \sum_{\mathbf{x} \in \cluster_{j_1,j_2, \dots  j_\nsets} }  \| \mathbf{x} -\protocentroid^{j_1}_1 \aggregator \protocentroid^{j_2}_2 \dots  \aggregator \protocentroid^{j_\nsets}_\nsets \|^2.
\end{equation*}
\end{problem} 

The classic \kmeansprob problem is known to be \np-Hard~\cite{dasgupta2008hardness}. Any instance of \kmeansprob can be mapped to an instance of \ourkmeans simply by setting $\nsets=1$ and $\card_1 = \nclusters$. Thus, Problem~\ref{prob:kr_kmeans_problem} is also \np-hard.

In view of this hardness, as it is customary for clustering methods, in Section~\ref{sec:algorithms_shallow}, 
we devise effective algorithms \clusteringfunction that do not seek a globally optimal solution.
%we devise algorithms \clusteringfunction for the problem that are effective but do not seek a globally-optimal solution. 
Figure~\ref{fig:diagramkr} (\textbf{a})  provides a schematic representation of the \ourkmeans problem. 

\subsection{\kr Deep Clustering}\label{sec:deepclusteringproblem} As explained in Section~\ref{sec:preliminaries}, centroid-based deep clustering algorithms \clusteringfunction map datasets \dataset to summaries in the form $\allparams = \allparamsautoencoder \cup \allparamscentroids$, where $\allparamsautoencoder$ and $\allparamscentroids$ are autoencoder and centroid parameters, respectively.

The \kr extension of the deep clustering problem adopts the same objective function \clusteringfunctionquality, while enforcing  constraints on \allparams to craft a succinct representation. As in \kr \kmeansprob, (latent-space) centroids are conjectured to follow the \kr structure, i.e.: 
$\allparams_\mu =\{  \protocentroid^{j_i}_l \mid j_i \in [1, \card_l]~\text{ and }~l \in [1,p] \}$.
In addition, we also enforce a succinct representation of the autoencoder parameters.

In the autoencoder-based deep clustering setting, autoencoders consist of an \emph{encoder} $f^e_{\allparamsautoencoder}  : \dataset \to Z$  mapping data to latent representations $Z = \{\mathbf{z}_1, \dots, \mathbf{z}_n\}$, and a \emph{decoder} $f^d_{\allparamsautoencoder}  : Z \to \hat{\dataset}$, reconstructing data from their latent representations. We consider an autoencoder with $\nlayers$ layers. The $l$-th layer is parameterized by a matrix $\autoencoderweightmatrixl \in \mathbb{R}^{d_l \times \nfeats_l}$, which could be relatively large. To obtain a succinct representation of the output of deep clustering algorithms, it is thus not sufficient to compress the centroid parameters, but it is also necessary to compress the autoencoder parameters $\allparamsautoencoder = \{ \autoencoderweightmatrixl \}_{l=1}^{\nlayers}$. To this aim, we draw on the connection between \kr clustering and Hadamard decomposition~\cite{ciaperoni2024hadamard} discussed in Section~\ref{sec:related}, and we reparametrize matrices $\autoencoderweightmatrixl$ as Hadamard products of \nhadamardfactors  factors, i.e.:
\begin{equation}\label{eq:hadamard_repa}
\autoencoderweightmatrixl = (\amat^l_1 \bmat^l_1) \odot (\amat^l_2 \bmat^l_2) \dots (\amat^l_\nhadamardfactors \bmat^l_\nhadamardfactors) \quad \ \forall l \in [1, \nlayers], 
\end{equation}
where $\amat^l_i \in \mathbb{R}^{d_l \times \actualrank_i}$ and $\bmat^l_i \in \mathbb{R}^{\actualrank_i \times \nfeats_l}$. The key intuition behind such reparameterization mirrors that of \kr clustering. The Hadamard product of \nhadamardfactors matrices of ranks $\actualrank_1, \actualrank_2, \dots$ and $\actualrank_\nhadamardfactors$ can represent matrices of rank up to $\prod_{l=1}^\nhadamardfactors \actualrank_l$ while using only $2 \sum_{l=1}^\nhadamardfactors  \actualrank_l$
vectors. In practice, this provides an effective compression mechanism, as also demonstrated in the context of computer vision in federated learning environments~\cite{hyeon2021fedpara}, where a similar reparameterization is considered. 

Having illustrated the constraints that \kr deep clustering poses on \allparamscentroids and \allparamsautoencoder,  we can now formally state the \kr deep clustering problem. 

\begin{problem}\label{prob:krdeepclustering}
Given a dataset $\dataset$, a parametric mapping $f^e_{\allparamsautoencoder}  : \dataset \to Z$ to latent representations $Z = \{\mathbf{z}_1, \dots, \mathbf{z}_n\}$, a mapping $f^d_{\allparamsautoencoder}  : Z \to \hat{\dataset}$ back to the original feature space and input $p \in \mathbb{N}_+$, partition $\dataset$ into $  \card_1 \card_2 \dots \card_\nsets$ clusters so as to minimize
\begin{align*}
\label{eq:kr_deepclusteringObjective}
\mathcal{Q}_{C}(\dataset, \allparams) = \mathcal{L}_{\mathrm{cluster}}(Z, \allparamscentroids) + w_{rec}  \mathcal{L}_{\mathrm{rec}}(\dataset, f^d_{\allparamsautoencoder}(Z)),
\end{align*}
where $\allparams_\mu = \{ \protocentroid^{j_i}_l \mid j_i \in [1, h_l] ~\text{and}~ l \in [1, \nsets]\}$ and $ \allparamsautoencoder = \{ \amat^i_j, \bmat^i_j \mid i \in [1, \nlayers]~\text{and}~j \in [1, \nhadamardfactors]\}.$
\end{problem}

Considering the loss $\mathcal{L}_{cluster}$ in Equation~\eqref{eq:dkm_loss} optimized by \dkm, it can be defined as follows for \kr \dkm:
\begin{align*}
\begin{split}
        \mathcal{L}_{\mathrm{KR-DKM}}(\dataset, \allparamscentroids) = \frac{1}{\ndatapoints} &\sum_{z\in Z} 
    \sum_{j_1=1}^{\card_1} \dots \sum_{j_\nsets=1}^{\card_\nsets}
    {||{\bf z} - \protocentroid^{j_1}_1
   \dots \aggregator \protocentroid^{j_\nsets}_\nsets||}^2 \cdot \\
   & \frac{e^{-a{||{\bf z} - \protocentroid^{j_1}_1 
   \dots \aggregator \protocentroid^{j_\nsets}_\nsets||}^2}}{\sum_{l_1=1}^{\card_1} \dots \sum_{l_\nsets=1}^{\card_\nsets} e^{-a{||{\bf z} - \protocentroid^{l_1}_1 
   \dots \aggregator \protocentroid^{l_\nsets}_\nsets||}^2}}.
\end{split}
\end{align*}
Adjusting the loss $\mathcal{L}_{cluster}$ in Equation \eqref{eq:idec_loss} for the \idec algorithm to the \kr  paradigm follows the same logic. Figure~\ref{fig:diagramkr} (\textbf{b}) summarizes the  \kr deep clustering problem. Our solution to the problem is presented in Section~\ref{sec:deep_clustering}.  

\begin{figure}
    \centering
    %\begin{tabular}{c|c}
    % \resizebox{0.4\linewidth}{!}{\input{figures/khatri_rao_diagram_kmeans.tikz}} &  
    %\resizebox{0.5\linewidth}{!}{\input{figures/khatri_rao_diagram_deep.tikz}}  \\
    %     \textbf{(a)} &  \textbf{(b)}
    %\end{tabular}
     \includegraphics[width=0.45\textwidth]{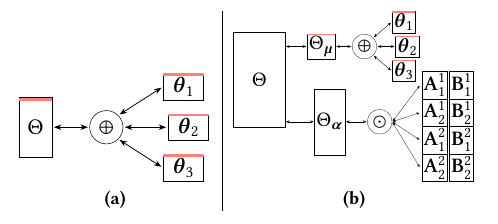}
    \caption{Diagram summarizing the Khatri-Rao clustering paradigm in the \kmeans \textbf{(a)} and deep clustering \textbf{(b)} settings; in this example, $\nsets = 3$, $\nlayers=2$ and $\nhadamardfactors=2$. The centroid in red is obtained by aggregating the \semicentroidsword in red.}
    \label{fig:diagramkr}
\end{figure}

\section{Naïve Khatri-Rao Clustering}\label{sec:naive}

As anticipated in Section~\ref{sec:introduction}, a naïve approach 
to Khatri-Rao clustering first applies a standard clustering algorithm (e.g., \kmeans) to obtain an initial set of centroids, and then post-processes these centroids to impose the succinct \kr structure.

In this section, we describe such a naïve solution to Khatri-Rao clustering. As for the rest of the algorithms we discuss, for clarity of exposition, we consider the simple case of two sets of \semicentroidsword, i.e., $\nsets=2$. The extension 
to the general formulation is straightforward but more cumbersome in notation. 

Let \centroidsset denote the set of centroids output by a centroid-based clustering algorithm like \kmeans. Given \centroidsset, we aim to find the closest sets of \semicentroidsword that yield \centroidsset upon aggregation, which translates into the following optimization problem:
\begin{equation*}
\label{eq:naiveproblem}        \min_{\{\protocentroid^i_1\}_{i=1}^{\card_1},\, \{\protocentroid^j_2\}_{j=1}^{\card_2} } \sum_{i,j} \| \centroid_{i,j} - \protocentroid^i_1 \aggregator \protocentroid^j_2 \|^2. 
\end{equation*}

This problem formulation suggests that \semicentroidsword could potentially be estimated via an alternating gradient-descent procedure where, at each iteration, \semicentroidsword from one set are updated by taking a step in the opposite direction of the gradient.

If the aggregator function \aggregator is the sum or the product, gradients can be computed in closed form. For example, assume that \aggregator is the product. Then, at any given iteration, the gradient with respect to, e.g., the $i$-th protocentroid  in the first set $\protocentroid_1^i$ would be
\begin{equation}
\label{eq:gradient_naive}
    2 \sum_{j=1}^{\card_2} 
    \left( 
        \centroid_{i,j} - \protocentroid^i_1 \odot \protocentroid^j_2 
    \right)
    \odot 
    \protocentroid^j_2, 
\end{equation}
where we have used the fact that only $\card_2$ centroids are affected by $\protocentroid_1^i$ (i.e., those corresponding to the combination of $\protocentroid_1^i$ with all \semicentroidsword in the second set), and hence contribute to the gradient. The gradient in Equation~\eqref{eq:gradient_naive} implies that, at any given iteration, the update rule for the  $i$-th \semicentroidword in the first set  $\protocentroid_1^i$, obtained by equating the corresponding gradient to \zerovector, is
\begin{equation}
\label{eq:update_naive}
    \protocentroid^i_1 
    \gets 
    \frac{
        \sum_{j=1}^{\card_2} 
        \centroid_{i,j} \odot \protocentroid^j_2
    }{
        \sum_{j=1}^{\card_2} 
        \protocentroid^j_2 \odot \protocentroid^j_2
    }.
\end{equation}
Update rules for the sum aggregator and alternative configurations of \semicentroidword sets can be obtained analogously. 

\spara{Limitations.}
While the gradient-descent algorithm described above can be effective for estimating \semicentroidsword associated with a given set of centroids, the performance of the two-phase naïve approach to Khatri-Rao clustering is often far from satisfactory. This is because the centroids obtained in the first step may accurately describe the dataset, yet be arbitrarily far from a \kr structure. As a consequence, the accuracy of the data summary found in the first phase can be destroyed in the second phase. 

Rather than using a two-phase approach, we need algorithms that simultaneously optimize  clustering and enforce the Khatri-Rao structure. The next section illustrates one such algorithm.

\section{The \hadkmeans Algorithm}\label{sec:algorithms_shallow}

In this section, we present the \hadkmeans algorithm, an extension of the standard \kmeans algorithm that addresses Problem~\ref{prob:kr_kmeans_problem} while overcoming the limitations of the naïve approach. The key steps of \hadkmeans are similar, in spirit, to those of standard \kmeans outlined in Section~\ref{sec:preliminaries}. However, \hadkmeans leverages the Khatri-Rao structure to achieve a more succinct representation of the centroids, which requires modifying the steps of standard \kmeans. 

\begin{algorithm}[t] 
\caption{\hadkmeans \\ \hspace*{1.8cm}(for two sets of \semicentroidsword, i.e., $\nsets=2$)}
\label{alg:hadkmeans}
\begin{algorithmic}[1]
\State \textbf{Input:} dataset $\dataset = \{ \xbf_i \}_{i=1}^\ndatapoints $, \semicentroidword sets cardinalities $\card_1, \card_2$, maximum number of iterations $n_{iter}$, tolerance $\tol$.
\State \textbf{Output:} assignments $\cluster_{\cdot,\cdot}$, \semicentroidsword $\protocentroids_1, \protocentroids_2$. 
\State $\protocentroid_1 \gets \text{SampleRandomly}(\dataset, \card_1)$\Comment{sample $\card_1$ points from $\dataset$}
\State $\protocentroid_2 \gets \text{SampleRandomly}(\dataset, \card_2)$ \Comment{sample $\card_2$ points from $\dataset$}
\State $\protocentroids_1^{\text{old}} \gets \protocentroids_1$, \ $\protocentroids_2^{\text{old}} \gets \protocentroids_2$
\For{$it = 1$ \textbf{to} $n_{iter}$}
    \State $\abf_1 \gets \mathbf{0}$, \ $\abf_2 \gets \mathbf{0}$, \ $\mathbf{d}_{\min} \gets \mathbf{\infty}$
    \For{$i = 1$ \textbf{to} $\card_1$}
        \For{$j = 1$ \textbf{to} $\card_2$}
            \State $\boldsymbol{\mu}_{ij} \gets \protocentroids_1^i \aggregator \protocentroids_2^j$\Comment{centroid computed on the fly} 
            \State $\mathbf{d}_{ij} \gets \|\dataset - \boldsymbol{\mu}_{ij}\|^2$\label{line:bottleneck} \Comment{compute distances}
            \State  $\mathbf{F} \gets (\mathbf{d}_{ij} < \mathbf{d}_{\min})$ \Comment{flag relevant samples}
            \State  $\abf_1[\mathbf{F}] \gets i$, \ $\abf_2[\mathbf{F}] \gets j$ \Comment{update assignments}
            \State  $\mathbf{d}_{\min}[\mathbf{F}] \gets \mathbf{d}_{ij}[\mathbf{F}]$
            \EndFor
    \EndFor
    \State $\cluster_{i,j} \gets \{\xbf_t \mid t \in \{1,2\dots n\} \wedge i = \abf_{1_t} \wedge j = \abf_{2_t}\}$
    
    \For{$i = 1$ \textbf{to} $\card_1$}\Comment{update \semicentroidsword}
        \State $\protocentroid_1^i\gets 
        \argmin_{\protocentroid_1} \sum_{l=1}^{\card_2} 
       \sum_{\xbf \in \cluster_{i,l}} (\xbf - \protocentroid_1  \aggregator \protocentroid_2^l)^2$  
       \label{line:bottleneck1} 
    \EndFor
    \For{$j = 1$ \textbf{to} $\card_2$}
        \State $\protocentroid_2^j \gets 
        \argmin_{\protocentroid_2} \sum_{l=1}^{\card_1} 
       \sum_{\xbf \in \cluster_{l,j}} (\xbf - \protocentroid_1^l  \aggregator \protocentroid_2)^2$ \label{line:bottleneck2}
    \EndFor
    \State $ \Delta \gets \sum_{i=1}^{\card_1} \sum_{j=1}^{\card_2} \left\| \protocentroids_1^i \;\oplus\; \protocentroids_2^j \;-\; \protocentroids_{1}^{\text{old},i} \;\oplus\; \protocentroids_{2}^{\text{old},j} \right\|^2 $
    \If{$\Delta < \tol$}\Comment{check stopping condition}
        \State \textbf{break}
    \EndIf
    %\State \textbf{if} $\Delta < \tol$ \textbf{break} \Comment{check stopping condition}
    \State $\protocentroids_1^{\text{old}} \gets \protocentroids_1$, \ $\protocentroids_2^{\text{old}} \gets \protocentroids_2$
\EndFor       
\State \textbf{Return:} assignments $\cluster_{\cdot,\cdot}$ and \semicentroidsword $\protocentroids_1, \protocentroids_2$. 
\end{algorithmic}
\end{algorithm}

A detailed description of the steps of \hadkmeans follows, and Algorithm~\ref{alg:hadkmeans} summarizes them. For clarity, we again focus on the scenario where we have two sets of \semicentroidsword. 

\spara{Initialization.}
Instead of centroids, \hadkmeans starts by initializing \semicentroidsword. As for standard \kmeans, a simple initialization strategy involves choosing initial \semicentroidsword by sampling data points uniformly at random. Alternatively, we can adapt the more effective initialization of \kmeanspp to be compatible with \hadkmeans. To achieve this, we sample $\card_1 + \card_2$ centroids $\centroid_{j_1, j_2}$ such that they are far from each other, and for each sampled centroid, we generate two \semicentroidsword such that 
\(
\centroid_{j_1, j_2} = \protocentroid^{j_1}_1 \aggregator \protocentroid^{j_2}_2  
\). 
Because of the constraints imposed in Khatri-Rao clustering, the remaining initial $\card_1 \card_2  - (\card_1 + \card_2)$ centroids are determined by the choice of the first $\card_1 + \card_2$. For simplicity, in Algorithm~\ref{alg:hadkmeans},  random sampling is used at the initialization stage. 

\spara{Centroid computation.}
The centroids are obtained on the fly at each iteration by simply aggregating \semicentroidsword according to the chosen aggregator function. 

\spara{Assignment update.}
Following the update of the centroids, at each iteration, just like \kmeans, \hadkmeans updates cluster assignments based on the latest set of centroids. This step is accomplished by assigning each observation to the centroid it is closest to according to the Euclidean distance. 

Since each centroid index is uniquely associated with a tuple of protocentroid indices, the assignments of data points to \semicentroidsword readily follow. 

\spara{Protocentroid update.}
The standard \kmeans algorithm updates cluster centroids by computing cluster means since the cluster means minimize the sum of squared distances to the centroid within each cluster (i.e., the cluster-specific contribution to the total inertia) based on the latest cluster assignments. Likewise, \hadkmeans updates \semicentroidsword by considering the same optimization. Nevertheless, as a consequence of the constraints imposed on the centroids, the optimal updates for \hadkmeans are not obtained by merely averaging, as stated in Proposition~\ref{prop:closed_form_updates} (proof in Appendix~C) for the sum and product aggregators.

\begin{proposition}\label{prop:closed_form_updates}
The optimal updates of the $j$-th \semicentroidword in the first and second set of \semicentroidsword at any iteration of \hadkmeans are given by 
\begin{align*}
       \protocentroid^{j}_1 = 
       \frac{
       \sum_{l=1}^{\card_2} 
       \sum_{\xbf \in \cluster_{j,l}} \xbf \odot \protocentroid^l_2 } 
        {
\sum_{l=1}^{\card_2} 
|\cluster_{j,l}|
       {\protocentroid^l_2} \odot {\protocentroid^l_2}
        }
         \text{ and }
        \protocentroid^{j}_2 = 
       \frac{
       \sum_{l=1}^{\card_1} 
       \sum_{\xbf \in \cluster_{l,j}} \xbf \odot \protocentroid^l_1 } 
        {
        \sum_{l=1}^{\card_1}
       |\cluster_{l,j}|
       {\protocentroid^l_1} \odot {\protocentroid^l_1}
        },
\end{align*} 
if $\aggregator = \times$, or if 
$\aggregator = +$: 
\begin{align*}
       \protocentroid^{j}_1 = 
       \frac{
       \sum_{l=1}^{\card_2} 
       \sum_{\xbf \in \cluster_{j,l}} (\xbf - \protocentroid^l_2) } 
        {
        \sum_{l=1}^{\card_2} 
        | \cluster_{j,l} | 
        }
         \text{ and }
        \protocentroid^{j}_2 = 
       \frac{
       \sum_{l=1}^{\card_1} 
       \sum_{\xbf \in \cluster_{l,j}} (\xbf - \protocentroid^l_1) } 
        {
        \sum_{l=1}^{\card_1} 
        | \cluster_{l,j} | 
        }. 
\end{align*}
\end{proposition}

\spara{Termination.}
Like \kmeans, \hadkmeans stops when either the protocentroids move less than a tolerance threshold $\tol$ or a maximum number of iterations $n_{iter}$ is reached. 

\spara{Complexity.}
The updates of the \semicentroidsword, i.e., Line~\ref{line:bottleneck1} and Line~\ref{line:bottleneck2} in Algorithm~\ref{alg:hadkmeans}, can be efficiently computed in closed form as per Proposition~\ref{prop:closed_form_updates}. In practice, the computations are sped up by keeping track of the assignments of each data point for both sets of \semicentroidsword and only considering points that are assigned to the \semicentroidword to update. 

Like for the standard \kmeans algorithm, the computational bottleneck is the computation of the distances for the assignment in Line~\ref{line:bottleneck}. \hadkmeans with $\card_1 + \card_2$ \semicentroidsword has the same time complexity as \kmeans with $\card_1 \card_2$ centroids, namely $\bigO( \ndatapoints \nfeats \card_1 \card_2 )$ per iteration. Considering dataset storage, \hadkmeans requires $\mathcal{O}\big((\ndatapoints +  \card_1 + \card_2) \nfeats\big)$ space, which is less than the $\mathcal{O}( (\ndatapoints + \card_1 \card_2) \nfeats)$ required by \kmeans, as long as $\card_1$ and $\card_2$ are larger than $2$. Hence, \hadkmeans can be more space-efficient than \kmeans in applications with a large number of clusters. If memory is not a concern, one can opt for a time-efficient implementation of \hadkmeans which stores the full set of $\card_1 \card_2$ centroids (implementation details in Appendix~B). Section~\ref{sec:scalibility} presents an empirical scalability analysis that corroborates the discussion of time and space complexity.

\spara{Limitations.}
As demonstrated empirically in Section~\ref{sec:experiments}, \hadkmeans\
can provide a more accurate summary of the data than standard \kmeans which uses the same number of parameters. However, \hadkmeans has a significant limitation. As it can also be seen from Figures~\ref{fig:interactiondiagram}~and~\ref{fig:example_centroids}, given two sets of $\card_1$ and $\card_2$ \semicentroidsword, each update of a \semicentroidword in the first set (Line~\ref{line:bottleneck1}) affects the position of $\card_2$ \semicentroidsword ~in the second set (Line~\ref{line:bottleneck2}), and vice versa. In the standard \kmeans algorithm, instead, centroids for a given assignment are updated independently of each other. The additional rigidity of \hadkmeans makes it more likely to converge to poor local minima compared to standard \kmeans. 

To overcome the lack of flexibility of \hadkmeans, we start from \hadkmeans and turn to deep clustering, relying on representation learning to perform clustering in a latent space which exhibits both a strong cluster structure and the Khatri-Rao structure, as discussed in the following section.

\section{Khatri-Rao Deep Clustering Algorithms}\label{sec:deep_clustering}

The Khatri-Rao deep clustering problem (Problem~\ref{prob:krdeepclustering}) requires optimizing a deep-clustering loss function \clusteringfunctionquality while compressing the centroid parameters $\allparams_{\centroid}$ and the autoencoder parameters $\allparamsautoencoder$ of a deep clustering algorithm. Unlike for \hadkmeans, addressing the Khatri-Rao deep clustering problem does not require the introduction of a completely novel machinery. The Khatri-Rao deep clustering framework extends standard deep clustering without requiring major adjustments. 

\spara{Initialization.}
The initialization of the centroids in the latent space plays an important role in driving the performance of deep clustering algorithms. To find a high-quality set of initial centroids, many deep clustering algorithms, including \dkm and \idec, rely on algorithms like \kmeans. Therefore, for the same goal, it is natural for algorithms based on the \krdcfull framework to use \hadkmeans.  

\spara{Reparameterization.}
From an optimization standpoint, the trainable parameters in deep clustering are the autoencoder and centroid parameters. To ensure that the constraints on those parameters placed by the Khatri-Rao deep clustering problem are met, the Khatri-Rao deep clustering framework reparameterizes standard deep clustering imposing the Hadamard-decomposition structure on the autoencoder parameters and the Khatri-Rao structure on the centroid parameters. 
As in standard deep clustering, all parameters are optimized via batch-wise backpropagation. 

From a computational standpoint, \kr deep clustering reduces trainable parameters at the cost of few additional operations.

Our experiments reveal that the \krdcfull framework reduces the size of data summaries found by standard deep clustering algorithms by more than $50\%$ on average in the datasets we consider, at little or no cost in accuracy.

\section{Design Choices in \kr Clustering}\label{sec:designchoices}

In standard clustering, the number of clusters and centroids is either specified based on domain knowledge, determined by practical constraints (e.g., the available memory for storing centroids), or estimated through data-driven procedures. Given a desired number of centroids to be represented, \kr clustering requires choosing the cardinality  of each set of \semicentroidsword, the number of sets and the aggregator function. \kr deep clustering also requires choosing the number of Hadamard factors and the rank of each factor. In this section, we address these choices, and we also explain how \kr clustering can integrate existing techniques to estimate the number of clusters.

\spara{Choosing the cardinality of sets of protocentroids.}
To maximize the number of centroids that can potentially be represented, it is convenient to consider \semicentroidword sets that are as balanced as possible. In general, $\nsets$ sets of $\card_1, \card_2, \dots \card_\nsets$ \semicentroidsword can represent $\prod_{i=1}^\nsets \card_i$ centroids. Given a budget $\budget$ of $\nsets \card$ vectors, allocating them in $\nsets$ sets can represent $\card^\nsets$ centroids. Any other allocation results in the potential representation of $\prod_{i=1}^\nsets \card_i < \card^\nsets$ centroids. 

Possible advantages over standard clustering arise whenever $\prod_{i=1}^\nsets \card_i > \sum_{i=1}^\nsets \card_i$. For instance, two sets of two \semicentroidsword can represent four centroids, yielding no advantage. 

\spara{Choosing the number of sets of protocentroids.}
In all Khatri-Rao clustering algorithms, the number \nsets of \semicentroidsword plays an important role. In this work, we primarily focus on the case of $\nsets=2$, which we recommend as the default choice. Increasing the value of \nsets renders the optimization problem more challenging and hinders interpretability. However, larger compression gains may in principle be obtained by considering $\nsets>2$. For example, given a budget of $12$ vectors to represent centroids, allocating them in $2$ and $3$ sets of equal size can represent $36$ and $64$ centroids, respectively. As explained in the previous paragraph, sets of equal size maximize the number of centroids that can be represented.  In addition, Proposition~\ref{prop:maximizer}, proved in Appendix~C, 
characterizes the value of $\nsets$ maximizing the number of centroids that can be represented. 
\begin{proposition}\label{prop:maximizer}
    Given a fixed budget of vectors $\budget$ to represent centroids, among possible divisors of \budget, the number $\nsets^{max}$ of \semicentroidword sets of equal size that maximizes the number of centroids that can be represented is one of the two divisors of $\budget$ that are closest to $\frac{\budget}{e}$, where $e$ is the natural logarithm base.
\end{proposition}

It is also possible to bound the number of sets that are guaranteed to represent $\nclusters$ centroids, as formalized in Proposition~\ref{prop:number_of_sets_bound}. 

\begin{proposition}
\label{prop:number_of_sets_bound}
Let $\card_{min}$ be the minimum number of \semicentroidsword in each set. The number $\nsets^*$ of sets of \semicentroidsword that are guaranteed to represent $\nclusters$ centroids satisfies 
\begin{align*}
     \log_{\card_{min}} \nclusters  \leq \nsets^* \leq \lceil \frac{\nclusters}{\card_{min}-1} \rceil. 
\end{align*}
\end{proposition}
The proof is given in Appendix~C.

\spara{Choosing the number of centroids.}
If the budget $\budget$ of vectors is not fixed and one wants to identify the number of clusters that are appropriate for the data, \kr clustering can be combined with established techniques such as \ensuremath{X}-\textsc{Means}~\cite{Pelleg2000Xmeans} or \ensuremath{G}-\textsc{Means}~\cite{Hamerly2003Learning}. Here, the number of centroids is successively increased and the current parameterization is evaluated, e.g., by using the Bayesian Information Criterion~\cite{schwarz1978estimating} or by testing if certain distributional conditions are fulfilled. In \kr clustering, increasing the number of clusters is equivalent to either increasing the cardinality of one set of protocentroids or the number of sets of protocentroids. Moreover, for deep clustering,  a \kr variant of such algorithms as \textsc{DipDECK}~\cite{leiber2021dip} or \textsc{DeepDPM}~\cite{Ronen2022deepdpm}, which also optimize the number of underlying clusters, can be designed. 

\spara{Choosing the aggregator function.}
In this work, we focus on the sum and product aggregator functions. Unfortunately, deciding between the sum and product aggregators prior to running our algorithms is difficult, particularly when moving beyond the naïve two-phase procedure described in Section~\ref{sec:naive}, which relies on an initial set of unconstrained centroids. More specifically, given an initial set of unconstrained centroids, in the additive model, centroid differences $\centroid_{i,j} - \centroid_{i',j} = (\protocentroid_i + \protocentroid_j) - (\protocentroid_{i'} + \protocentroid_j)$ remain nearly constant across $j$, and a similar invariance holds for the multiplicative model after taking logarithms. This can provide a simple heuristic for deciding between the two models. 
Without an initial set of centroids, similar heuristics cannot be considered. 
In our empirical evaluation, both the sum and product aggregators achieve competitive performance in \kmeansprob clustering. In contrast, in \kr deep clustering, the sum aggregator is generally preferable, as it results in an easier optimization.

%In our empirical evaluation, in \kmeansprob clustering, the sum and product aggregators provide competitive performance. In \kr deep clustering, the sum aggregator is generally preferable to the product as it results in an easier optimization.

The choice of aggregator function does not need to be limited to sum or product. Conceptually, the \kr clustering paradigm could accommodate arbitrary aggregator functions. However, there exists a tension between expressivity and optimization difficulty. As a consequence, increasing aggregator complexity not only compromises interpretability, but can also hurt performance. 

\spara{Choosing the number of Hadamard factors and their ranks.}
In \kr deep clustering, by default, we re\-parameterize the autoencoder weights using a two-factor Hadamard decomposition, i.e., we set $\nhadamardfactors = 2$ in Equation~\eqref{eq:hadamard_repa}, which ensures stable gradient-based optimization. Using more Hadamard factors can improve compression, but reduces stability. The ranks of both factors are set to be equal (to maximize the rank that can be captured) and large enough so that the compressed autoencoder incurs comparable reconstruction quality to the uncompressed one.

\section{Experiments}\label{sec:experiments}

In this section, we evaluate the performance of Khatri-Rao clustering against baselines considering both synthetic and real benchmark datasets. The experiments primarily aim to show that \kr clustering achieves the goal of reducing the size of a centroid-based data summary without compromising its accuracy. We also showcase the practical benefits of \kr clustering through case studies.

\begin{figure*}[t]
\centering
%    \begin{tabular}{cc|cc}
%    \includegraphics[width=0.23\linewidth]{figures/inertia_exp_blobs_withkm_max.pdf}    & \includegraphics[width=0.23\linewidth]{figures/inertia_exp_classification_withkm_max.pdf}   &
%    \includegraphics[width=0.23\linewidth]{figures/purity_exp_blobs_withkm_max.pdf}    & \includegraphics[width=0.23\linewidth]{figures/purity_exp_classification_withkm_max.pdf} 
%    \\
%    \blobs & \classification & \blobs & \classification \\ 
%    \end{tabular}
\includegraphics[width=1\textwidth]{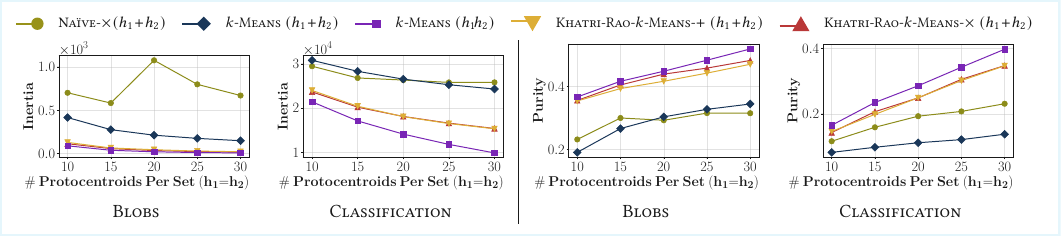}
\caption{Experiments using the \blobs and \classification datasets with $100$ ground truth clusters. Inertia (left) and purity (right) as a function of the cardinality $\card_1 = \card_2$ of two sets of \semicentroidsword. \kmeans$(\card_1 \card_2)$ uses $\card_1 \card_2$ vectors to represent centroids, while all other algorithms use $\card_1 + \card_2$ vectors. }
    \label{fig:purity_and_inertia}
\end{figure*}

\subsection{Settings}\label{sec:experimental_settings}
In the following, we present the experimental setup.

\spara{Datasets.}
We consider two synthetic dataset generators (\blobs and \classification) from the \textsc{scikit-learn} Python library~\cite{scikit-learn}, which allow exploring the behavior of different algorithms as a function of the number of data points, features and underlying clusters. The preliminary experiments and the scalability analysis focus on those datasets. 

In addition, we consider seven real-world benchmark datasets, available through the \textsc{ClustPy}~\cite{leiber2023benchmarking} Python library and two synthetic datasets available through the \textsc{clustbench}~\cite{gagolewski2022framework}  Python library. Finally, we consider the \double dataset, which is created by concatenating pairs of digits from \mnist~\cite{lecun1998gradient}, yielding $100$ clusters representing the numbers $0$ to $99$.

The salient characteristics of the datasets considered in the experimental evaluation are provided in Table~\ref{tab:datasets}. A detailed description of each dataset is provided in Appendix~A.  

\begin{table}[t]
\centering    
\caption{Characteristics of the datasets used in the experiments. For each dataset, we report the number of data points, the number of features, the number of ground-truth clusters (\# Labels) and the imbalance ratio (IR), defined as the ratio of the smallest to the largest cluster size.}
\footnotesize
\begin{tabular}{lcccc}
\toprule
Dataset & \# Data points & \# Features & \# Labels & IR \\
\midrule
\mnist & 25000 & 784 & 10 & 1.00 \\
\double & 10000 & 1568 & 100 & 1.00 \\
\har & 10299 & 561 & 6 & 0.72 \\
\olivetti & 400 & 4096 & 40 & 1.00 \\
\cmu & 624 & 960 & 20 & 0.88 \\
\symbols & 1020 & 398 & 6 & 0.90 \\
\stickfigures & 900 & 400 & 9 & 1.00 \\
\optdigits & 5620 & 64 & 10 & 0.97 \\
\classification & 5000 & 10 & 100 & 0.91 \\
\chameleon & 10000 & 2 & 10 & 0.10 \\
\soybean & 562 & 35 & 15 & 0.22 \\
\blobs & 5000 & 2 & 100 & 1.00 \\
\rfifteen & 600 & 2 & 15 & 1.00 \\
\bottomrule
\end{tabular}
\label{tab:datasets}
\end{table}

\spara{Parameter settings.}
For the purposes of our evaluation, we assume that the underlying ground-truth number of clusters in a dataset is known, and is given as input to all algorithms. Unless stated otherwise, for all experiments regarding Khatri-Rao clustering, we consider two sets of \semicentroidsword (i.e., $\nsets=2$). 
To maximize the number of centroids that can be represented, the cardinality of each set of \semicentroidsword is determined by selecting the two factors of the total number of clusters that are closest in value so that $\card_1\card_2=\nclusters$ (e.g., $\card_1=8$ and $\card_2=5$ for $\nclusters=40$). In all experiments, we run each method $20$ times with random initialization and select the solution giving the smallest inertia, i.e., the smallest squared Euclidean distance between samples and their assigned cluster centroids. 

\hadkmeans terminates when either $n_{iter}=200$ iterations are reached, or the centroid movement falls below $\epsilon=10^{-4}$.
We consider \hadkmeans with sum and product aggregator functions. Instead, in the deep clustering experiments we focus on the sum since it generally yields superior results, as explained in Section~\ref{sec:designchoices}.  

In all deep clustering experiments, the parameter $w_{rec}$, which balances the clustering and reconstruction losses as per Equation~\eqref{eq:deepclusteringObjective}, is set to $1$, the batch size for batch-wise gradient-based optimization is set to $512$, and we use the ADAM optimizer~\cite{Diederik2015adam} using a learning rate of $10^{-3}$ for pre-training the autoencoder and $10^{-4}$ for clustering.

Further, a preliminary step in the deep clustering experiments is required to learn two autoencoders, a fully-connected autoencoder for standard deep clustering algorithms and a compressed autoencoder for their Khatri-Rao variants. Autoencoders have an encoder with $5$ layers of dimensions $\nfeats{-}1024{-}512{-}256{-}10$ and the decoder is a mirrored version of the encoder. To ensure reliable performance of the Khatri-Rao deep clustering algorithms, it is crucial that the initial compressed autoencoder achieves a similar reconstruction quality as the uncompressed one. Given a matrix of weights $\autoencoderweightmatrixl \in \mathbb{R}^{d_l \times \nfeats_l}$, we initially set the rank of each of two Hadamard-decomposition factors to $\max\{ 10, \min \{  \nrows_l, \nfeats_l \} \}$, and we leave the input and output layers uncompressed, which  improves performance. If the loss of the initial compressed autoencoder is higher than that of the full autoencoder, we iteratively multiply the rank by $2,3, \dots$ until the loss of the compressed autoencoder falls under that of the full autoencoder. As shown in Section~\ref{sec:experimental_results}, this strategy guarantees remarkable parameter reductions. The number of epochs used for pre-training the standard autoencoders and for the actual clustering procedure is set to $150$. For pre-training the compressed autoencoder, we set the number of epochs to $1000$ and $500$ additional epochs are added whenever the rank is increased as explained above. 

\spara{Metrics.}
In our experiments, the accuracy of a data summary is measured by the quality of clustering results.
We consider the following standard metrics to evaluate clustering results that take advantage of ground truth labels: adjusted Rand index (ARI)~\cite{hubert1985comparing}, unsupervised clustering accuracy (ACC)~\cite{Yang2010accuracy} and normalized mutual information (NMI)~\cite{kvalseth1987entropy}. These metrics have a maximum value of $1$, which reveals a perfect match between predicted and ground-truth labels. In the case of \kmeansprob clustering, we also monitor inertia, which is the objective function optimized for by standard \kmeans and \hadkmeans. In addition, in the preliminary experiments, we monitor clustering purity, which measures the proportion of correctly assigned data points after assigning each data point to the majority ground-truth label of its cluster~\cite{manning2008introduction}.

To quantify the amount of compression, we monitor the number of parameters that a given algorithm uses to obtain a succinct summary of a dataset. 

\spara{Clustering algorithms.}
In the case of \kmeansprob clustering, we compare  \hadkmeans against the standard \kmeans algorithm (implementation from \textsc{scikit-learn}~\cite{scikit-learn}). Additionally, we consider the two-phase naïve approach described in Section~\ref{sec:naive} with product aggregator. When not clear from context, we specify the aggregator function and the budget of vectors used to represent centroids (e.g., \hadkmeanssum$(\card_1 + \card_2)$ denotes \hadkmeans with sum aggregator and $\card_1 + \card_2$ \semicentroidsword).

In the case of deep clustering, we consider \dkm and \idec (implementations from \textsc{ClustPy}~\cite{leiber2023benchmarking}) against their Khatri-Rao variants, \krdkm ~and \kridec. Appendix~B provides implementation details for all algorithms.

\spara{Computing environment.}
Experiments use a machine with a 13th Gen Intel Core i9-13900H processor (14 cores, 20 threads), an NVIDIA GeForce RTX 4060 GPU and 32 GB of RAM.

\begin{figure}[t]
  %  \hspace*{-0.35cm}\input{figures/newlegendby_p.tikz}
  %  \centering
  %  \begin{tabular}{cc}
   %  \includegraphics[width=0.46\linewidth]{figures/inertia_by_p_with_km_max3.pdf}    & \includegraphics[width=0.46\linewidth]{figures/inertia_by_p_with_km_max3_classification.pdf}   \\
   % \blobs   & \classification \\
   % \end{tabular}
   \includegraphics[width=0.46\textwidth]{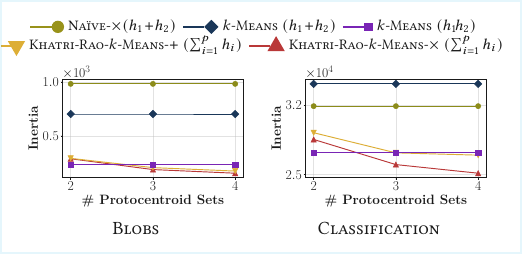} 
    \caption{Experiments using the \blobs and \classification datasets with $100$ ground truth clusters. Inertia as a function of the number of sets of \semicentroidsword. Here, \hadkmeans always uses $12$ total vectors to represent centroids. All baselines assume $\card_1 = \card_2 = 6$.}
    \label{fig:impact_of_p}
\end{figure}

\subsection{Quantitative Results}\label{sec:experimental_results}
In this section, we present the results of the experiments, first for the \kmeans setting and second for the deep clustering setting.

\begin{table*}[t]
\centering
\caption{\label{tab:kmeansclusteringresults} Experiments on synthetic and real-world datasets comparing \hadkmeans using the product (\hadkmeansprod) and sum (\hadkmeanssum) aggregator functions with two sets of $\card_1$ and $\card_2$ \semicentroidsword against \kmeans ($\card_1 + \card_2$) and \kmeans ($\card_1  \card_2$). For each dataset and algorithm, we report unsupervised clustering accuracy (ACC), adjusted Rand index (ARI), normalized mutual information (NMI) and inertia (normalized by dividing by the inertia of \kmeans ($\card_1  \card_2$)). The last column reports the ratio of the number of parameters used compared to \kmeans ($\card_1  \card_2$). 
}
\footnotesize
\resizebox{1\textwidth}{!}{
\begin{tabular}{l|cccccccccccccccccc|c}
\toprule
Dataset & \multicolumn{4}{c}{\hadkmeanssum($\card_1 + \card_2$)} & \multicolumn{4}{c}{\hadkmeansprod($\card_1 + \card_2$)} & \multicolumn{4}{c}{\kmeans ($\card_1 + \card_2$)} & \multicolumn{4}{c}{\kmeans ($\card_1  \card_2$)} & Params \\
\cmidrule(lr){2-5} \cmidrule(lr){6-9} \cmidrule(lr){10-13} \cmidrule(lr){14-17}
& ARI & ACC & NMI & Inertia & ARI & ACC & NMI & Inertia & ARI & ACC & NMI & Inertia & ARI & ACC & NMI & Inertia & \\
\midrule
\mnist & $0.298$ & $0.436$ & $0.428$ & $1.05$ & $0.309$ & $0.435$ & $0.441$ & $1.07$ & $0.357$ & $0.515$ & $0.467$ & $1.05$ & $0.366$ & $0.515$ & $0.498$ & $1.00$ & $0.70$ \\
\double & $0.109$ & $0.211$ & $0.467$ & $1.16$ & $0.068$ & $0.119$ & $0.396$ & $1.07$ & $0.049$ & $0.081$ & $0.329$ & $1.20$ & $0.075$ & $0.143$ & $0.428$ & $1.00$ & $0.20$ \\
\har & $0.320$ & $0.515$ & $0.486$ & $1.06$ & $0.285$ & $0.456$ & $0.451$ & $1.04$ & $0.291$ & $0.442$ & $0.461$ & $1.03$ & $0.420$ & $0.539$ & $0.559$ & $1.00$ & $0.83$ \\
\olivetti & $0.248$ & $0.398$ & $0.664$ & $1.34$ & $0.239$ & $0.393$ & $0.643$ & $1.41$ & $0.181$ & $0.287$ & $0.565$ & $1.46$ & $0.456$ & $0.600$ & $0.790$ & $1.00$ & $0.33$ \\
\cmu & $0.414$ & $0.503$ & $0.760$ & $1.52$ & $0.463$ & $0.561$ & $0.788$ & $1.52$ & $0.408$ & $0.452$ & $0.767$ & $1.46$ & $0.754$ & $0.772$ & $0.902$ & $1.00$ & $0.45$ \\
\symbols & $0.459$ & $0.605$ & $0.682$ & $1.52$ & $0.693$ & $0.768$ & $0.815$ & $1.35$ & $0.682$ & $0.758$ & $0.810$ & $1.22$ & $0.666$ & $0.689$ & $0.794$ & $1.00$ & $0.83$ \\
\stickfigures & $1.000$ & $1.000$ & $1.000$ & $1.00$ & $0.885$ & $0.889$ & $0.964$ & $4.79$ & $0.708$ & $0.667$ & $0.882$ & $11.57$ & $1.000$ & $1.000$ & $1.000$ & $1.00$ & $0.67$ \\
\optdigits & $0.477$ & $0.625$ & $0.622$ & $1.12$ & $0.457$ & $0.596$ & $0.589$ & $1.07$ & $0.465$ & $0.612$ & $0.622$ & $1.11$ & $0.491$ & $0.640$ & $0.648$ & $1.00$ & $0.70$ \\
\classification & $0.041$ & $0.134$ & $0.362$ & $1.12$ & $0.044$ & $0.137$ & $0.368$ & $1.10$ & $0.035$ & $0.083$ & $0.245$ & $1.44$ & $0.052$ & $0.159$ & $0.387$ & $1.00$ & $0.20$ \\
\chameleon & $0.318$ & $0.435$ & $0.551$ & $1.07$ & $0.307$ & $0.419$ & $0.545$ & $1.06$ & $0.373$ & $0.452$ & $0.583$ & $1.52$ & $0.296$ & $0.439$ & $0.538$ & $1.00$ & $0.70$ \\
\soybean & $0.331$ & $0.484$ & $0.647$ & $1.49$ & $0.368$ & $0.534$ & $0.654$ & $1.29$ & $0.329$ & $0.482$ & $0.637$ & $1.43$ & $0.406$ & $0.589$ & $0.674$ & $1.00$ & $0.53$ \\
\blobs & $0.236$ & $0.326$ & $0.655$ & $1.45$ & $0.242$ & $0.341$ & $0.656$ & $1.34$ & $0.207$ & $0.190$ & $0.655$ & $4.61$ & $0.238$ & $0.350$ & $0.658$ & $1.00$ & $0.20$ \\
\rfifteen & $0.787$ & $0.815$ & $0.910$ & $3.44$ & $0.919$ & $0.928$ & $0.970$ & $1.68$ & $0.264$ & $0.533$ & $0.743$ & $11.77$ & $0.993$ & $0.997$ & $0.994$ & $1.00$ & $0.53$ \\
\bottomrule
\end{tabular}
}
\end{table*}

\spara{Preliminary results.}
We begin by presenting preliminary experiments on the \blobs and \classification datasets. Figure~\ref{fig:purity_and_inertia} shows inertia and purity as a function of the cardinality ($\card_1 = \card_2$) of both sets of \semicentroidsword.  The inertia incurred by \hadkmeans is at most $31\%$ and $81\%$ of that incurred by any of the baselines using the same amount of parameters in the \blobs and \classification datasets, respectively. Similarly, the purity achieved by any such baselines is at most $76\%$ and $81\%$ of that achieved by \hadkmeans. Standard \kmeans can attain lower inertia and higher purity than \hadkmeans, but using $\card_1 \card_2$ vectors to represent centroids rather than $\card_1 + \card_2$.

The performance of \hadkmeans could be improved further by increasing the number \nsets of sets of \semicentroidsword. 
While we leave a thorough investigation to future work, Figure~\ref{fig:impact_of_p} compares the inertia incurred by \hadkmeans as a total of $12$ vectors that are split between $2$, $3$ and $4$ sets of \semicentroidsword against that incurred by baselines with $\card_1 = \card_2 = 6$. The results indicate that \hadkmeans with $12$ vectors to represent centroids can incur lower inertia than standard \kmeans with $\card_1 \card_2 = 36$ vectors as $\nsets$ increases. As discussed in Section~\ref{sec:designchoices}, increasing \nsets can improve performance but makes optimization more challenging. As a consequence, inertia decreases monotonically, but with diminishing reductions. In the remaining experiments, we focus on two sets of \semicentroidsword. Furthermore, the naïve approach is not competitive with \hadkmeans. Thus, in the remaining experiments, we focus on  comparing \hadkmeans against \kmeans.

\spara{Results for \kmeansprob.}
Table~\ref{tab:kmeansclusteringresults} summarizes the results of the experiments comparing \hadkmeans with two sets of $\card_1$ and $\card_2$ \semicentroidsword for sum and product aggregators against standard \kmeans with $\card_1 + \card_2$ and $\card_1 \card_2$ centroids. The table reports the adopted measures of clustering quality as well as the ratio of the inertia and the amount of parameters used for summarization relative to \kmeans with  $\card_1 \card_2$ centroids. The results indicate that \hadkmeans  often but not always outperforms standard \kmeans with the same number of parameters. 

The medians (means) of the inertia ratios reported in the table across datasets are $1.16$ ($1.41$), $1.29$ ($1.52$) and $1.44$ ($3.14$), for \hadkmeanssum, \hadkmeansprod and \kmeans respectively, with all algorithms using $\card_1 +\card_2$ vectors to represent centroids. Nonetheless, standard \kmeans with $\card_1 \card_2$ centroids generally provides superior performance to \hadkmeans, suggesting that \hadkmeans struggles to achieve its ideal performance also due to its lack of flexibility discussed in Section~\ref{sec:algorithms_shallow}.

\spara{Results for deep clustering.}
The experimental results comparing the deep clustering approaches \dkm against \krdkm and \idec against \kridec are shown in Table~\ref{tab:deepclusteringresults}. \dkm and \idec directly return $\card_1 \card_2$ clusters, while \krdkm and \kridec are working with two sets of $\card_1$ and $\card_2$ \semicentroidsword leading to the same number of clusters. The results demonstrate that the Khatri-Rao deep clustering framework can reduce the number of parameters by at least more than $10\%$ and as much as $85\%$ with a minor impact on performance. In some datasets, the performance of \kr deep clustering is even superior to that of standard deep clustering. It is conceivable that \kr deep clustering provides an implicit form of regularization. 

In summary, while \hadkmeans with $\card_1 +\card_2$ \semicentroidsword may not attain the same clustering quality as standard \kmeans with $\card_1 \card_2$ centroids, this is not the case for deep clustering. The performance of deep Khatri-Rao clustering algorithms is close to their optimistic bound.

In our experiments, only the \stickfigures and \double datasets are known to exhibit a \kr structure. However, \kr clustering can deliver advantages over traditional clustering, irrespective of whether its underlying model holds.

\begin{table*}[t]
\centering
\caption{\label{tab:deepclusteringresults}Experiments on synthetic and real-world datasets comparing the deep-clustering algorithms \dkm and \idec  to their Khatri-Rao counterparts using the sum aggregator function. For each dataset and algorithm, we report unsupervised clustering accuracy (ACC), adjusted Rand index (ARI) and normalized mutual information (NMI). The last column reports the ratio of the number of parameters used by \krdkm and \kridec compared to \dkm and \idec.}
\footnotesize
\begin{tabular}{lcccccc|cccccc|c}
\toprule
    Dataset & \multicolumn{3}{c}{\idec} & \multicolumn{3}{c}{\kridec} & \multicolumn{3}{c}{\dkm} & \multicolumn{3}{c}{\krdkm} & $\text{Params Ratio}$ \\
\cmidrule(lr){2-4} \cmidrule(lr){5-7} \cmidrule(lr){8-10} \cmidrule(lr){11-13}
& ARI & ACC & NMI & ARI & ACC & NMI & ARI & ACC & NMI & ARI & ACC & NMI & \\
\midrule
\mnist & $0.616$ & $0.771$ & $0.682$ & $0.630$ & $0.746$ & $0.712$ & $0.574$ & $0.718$ & $0.693$ & $0.645$ & $0.788$ & $0.719$ & $0.74$ \\
\double & $0.174$ & $0.285$ & $0.532$ & $0.180$ & $0.292$ & $0.536$ & $0.178$ & $0.295$ & $0.537$ & $0.201$ & $0.330$ & $0.553$ & $0.79$ \\
\har & $0.581$ & $0.692$ & $0.660$ & $0.621$ & $0.703$ & $0.701$ & $0.539$ & $0.649$ & $0.634$ & $0.619$ & $0.676$ & $0.714$ & $0.62$ \\
\olivetti & $0.399$ & $0.517$ & $0.746$ & $0.371$ & $0.510$ & $0.734$ & $0.437$ & $0.570$ & $0.769$ & $0.387$ & $0.542$ & $0.736$ & $0.88$ \\
\cmu & $0.564$ & $0.646$ & $0.782$ & $0.648$ & $0.707$ & $0.823$ & $0.690$ & $0.761$ & $0.852$ & $0.525$ & $0.631$ & $0.743$ & $0.66$ \\
\symbols & $0.655$ & $0.684$ & $0.779$ & $0.689$ & $0.722$ & $0.798$ & $0.662$ & $0.693$ & $0.786$ & $0.711$ & $0.746$ & $0.813$ & $0.47$ \\
\stickfigures & $1$ & $1$ & $1$ & $1$ & $1$ & $1$ & $1$ & $1$ & $1$ & $1$ & $1$ & $1$ & $0.47$ \\
\optdigits & $0.688$ & $0.809$ & $0.751$ & $0.645$ & $0.775$ & $0.721$ & $0.634$ & $0.733$ & $0.742$ & $0.681$ & $0.774$ & $0.766$ & $0.22$ \\
\classification & $0.032$ & $0.124$ & $0.336$ & $0.007$ & $0.070$ & $0.245$ & $0.044$ & $0.144$ & $0.369$ & $0.040$ & $0.133$ & $0.356$ & $0.16$ \\
\chameleon & $0.321$ & $0.466$ & $0.564$ & $0.312$ & $0.415$ & $0.550$ & $0.329$ & $0.474$ & $0.589$ & $0.319$ & $0.476$ & $0.540$ & $0.15$ \\
\soybean & $0.391$ & $0.562$ & $0.657$ & $0.451$ & $0.617$ & $0.696$ & $0.373$ & $0.546$ & $0.661$ & $0.445$ & $0.635$ & $0.693$ & $0.19$ \\
\blobs & $0.182$ & $0.214$ & $0.597$ & $0.188$ & $0.239$ & $0.609$ & $0.226$ & $0.314$ & $0.653$ & $0.215$ & $0.307$ & $0.650$ & $0.15$ \\
\rfifteen & $0.986$ & $0.993$ & $0.988$ & $0.989$ & $0.995$ & $0.991$ & $0.830$ & $0.850$ & $0.926$ & $0.993$ & $0.997$ & $0.994$ & $0.15$ \\
\bottomrule
\end{tabular}
\end{table*}

\subsection{Scalability}\label{sec:scalibility}

\begin{figure*}[t]
\centering
%\begin{tabular}{ccc}
%   \includegraphics[width=0.28\linewidth]{figures/runtime_by_data_points_with_km_max50_blobs.pdf}\hspace{0.3cm}  & 
%   \includegraphics[width=0.28\linewidth]{figures/runtime_by_features_with_km_max50_blobs.pdf} & 
%   \includegraphics[width=0.28\linewidth]{figures/runtime_by_centroids_with_km_max50_blobs.pdf} 
%   \\
%   \includegraphics[width=0.28\linewidth]{figures/memory_by_data_points_with_km_max50_blobs.pdf}  & 
%   \includegraphics[width=0.28\linewidth]{figures/memory_by_features_with_km_max50_blobs.pdf} &
 %  \includegraphics[width=0.28\linewidth]{figures/memory_by_centroids_with_km_max50_blobs.pdf}
%\end{tabular}
\includegraphics[width=1.03\textwidth]{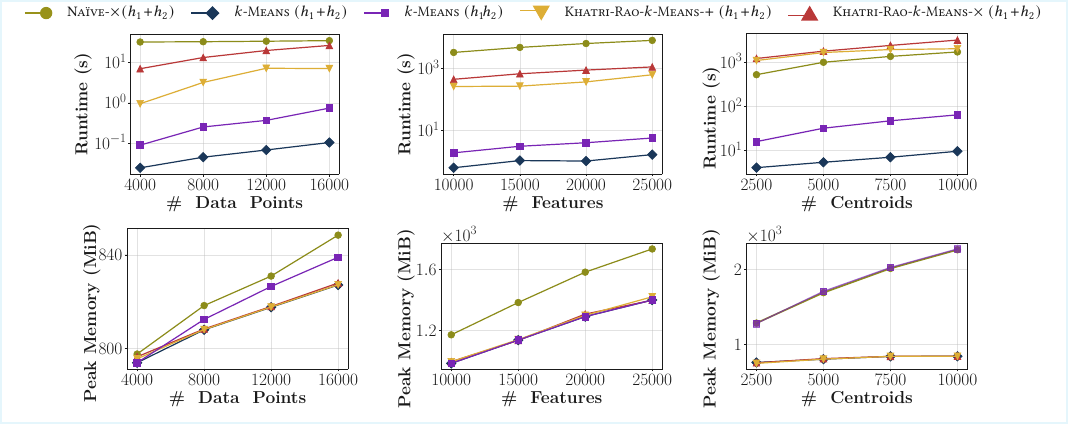}
\caption{Experiments using the \blobs dataset. Runtime in seconds (top) and peak memory usage in Mebibytes (bottom) by number of data points $\ndatapoints$, number of features $\nfeats$ and number of centroids $\nclusters$. Here, \kmeans$(\card_1 \card_2)$ uses $\card_1 \card_2$ vectors to represent centroids, while all other algorithms use $\card_1 + \card_2$ vectors. The $y$-axis is on a logarithmic scale.}
\label{fig:scalibility}
\end{figure*}

Our quantitative experiments suggest that \kr clustering can yield more succinct clustering-based summaries than standard clustering methods while maintaining a comparable accuracy. However, it is also important to ensure that the provided advantages are not undermined by significantly worse scalability. According to the discussion in Section~\ref{sec:algorithms_shallow}, \hadkmeans has the same asymptotic time complexity as \kmeans. In addition, \hadkmeans can reduce the memory usage of \kmeans when the number of clusters and hence centroids to be represented grows. To complement the analysis of time and space complexity, we conduct an empirical scalability analysis. 

The results of the scalability analysis are summarized in Figure~\ref{fig:scalibility}. This figure shows runtime (in seconds) and peak memory usage (in Mebibytes) for a single execution of an algorithm when increasing the number of data points, features and centroids using the \blobs dataset. In particular, we vary the number of data points in  $\{4000, 8000, 12000, 16000\}$ (with $100$ clusters and $100$ features), features in $\{10000 , 15000, 20000, 25000\}$ (with $1000$ data points and $100$ clusters), and clusters in $\{2500 , 5000, 7500, 10000\}$ (with $20000$ data points and $100$ features). 

The results suggest that \hadkmeans clustering introduces a runtime overhead compared to standard \kmeans. However, the overhead remains nearly constant as the number of data points, features and centroids increases, in line with the time-complexity discussion in Section~\ref{sec:algorithms_shallow}. The two-phase naïve approach to \ourkmeans is usually the slowest algorithm, except when the number of centroids becomes large, in which case \hadkmeans can become slower.

Considering memory requirements, \hadkmeans often exhibits a similar behavior as standard \kmeans using $\card_1 + \card_2$ vectors to represent centroids. Standard \kmeans with $\card_1 \card_2$ centroids, on the other hand, can be more memory demanding, particularly when the number of centroids grows, using up to $2.7$ times more memory than \hadkmeans. 
%\kmeans with $\card_1 \card_2$ centroids uses up to $2.7$ times more memory as \hadkmeans. 
Furthermore, the gap grows with the number of centroids, in agreement with the space-complexity discussion in Section~\ref{sec:algorithms_shallow}.

\subsection{Case Studies}

\begin{figure*}[t!]
    \centering
    \includegraphics[width=1\textwidth]{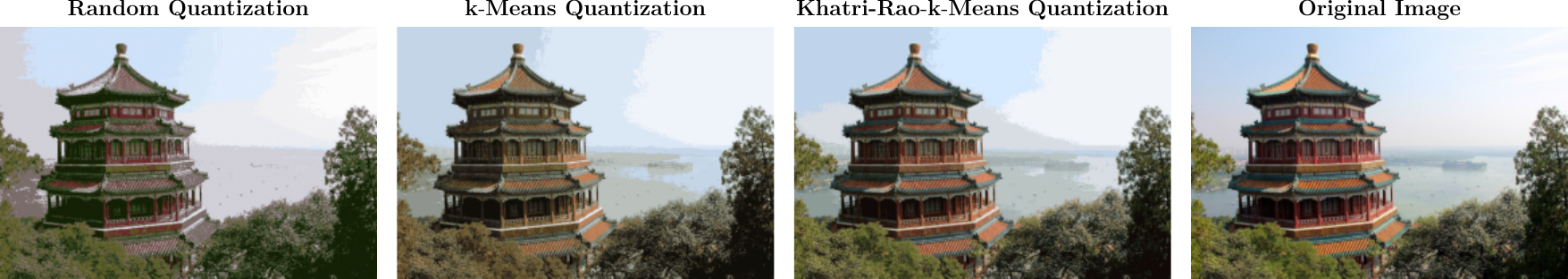}
    \caption{Comparison of \hadkmeans with product aggregator using two sets of $\card_1 = \card_2 = 6$ \semicentroidsword with different approaches to color quantization using the same number of parameters.}
\label{fig:combined_images}
\end{figure*}

\begin{figure}[t!]
\centering
 \includegraphics[width=0.45\textwidth]{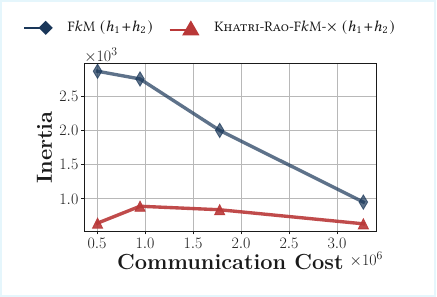}
\caption{\FEMNIST  dataset. Inertia against communication costs from the server to the clients (in Bytes) in a simulated federated learning environment for \fedkmeans and \fedhadkmeans using the product aggregator.}
\label{fig:federated}
\end{figure}

We conclude our experimental evaluation by presenting two simple case studies demonstrating the benefits of our \kr clustering paradigm when used for color quantization and in a federated-learning environment. The focus of the case studies is not to claim state-of-the-art performance on the two tasks, but to highlight the practical advantages that our paradigm can offer in real-world applications. Thus, for the purposes of the case studies,  we restrict our attention to \hadkmeans using the product aggregator.

\spara{Obtaining more succinct codebooks for color quantization.}
Color quantization in computer graphics enables efficient compression of certain types of images by reducing the number of colors. It allows displaying images with many colors on hardware-constrained devices that can only display a limited number of colors, usually due to memory limitations~\cite{brun2017color}. A common approach casts color quantization as clustering in a three-dimensional RGB space, where pixels are points. The centroids obtained by \kmeans form a codebook of representative colors, and each pixel is mapped to its closest centroid for color quantization.

\hadkmeans has the potential to extract more succinct and yet equally accurate codebooks compared to standard \kmeans. To investigate this potential, we consider the image from the \textsc{scikit-learn} `Color Quantization using K-Means' example\footnote{\url{https://scikit-learn.org/0.19/auto_examples/cluster/plot_color_quantization.html}}. We run \kmeans and \hadkmeans with the same number of centroid parameters on a subset of the image ($1000$ pixels), and compare the obtained color quantizations. In particular, \kmeans uses $12$ centroids, and \hadkmeans two sets of $6$ \semicentroidsword. For reference, we also include a color-quantization strategy which picks $12$ pixels uniformly at random to form the codebook. 

The results are given in Figure~\ref{fig:combined_images}, and verify that \hadkmeans can outperform standard \kmeans in the task of color quantization as it better preserves the colors of the original image. This is particularly evident in the improved representation of red tones. This visual assessment is confirmed numerically as the values of the inertia incurred by random quantization, \kmeans and \hadkmeans are $4686$, $2009$ and $1144$, respectively.

\spara{Reducing communication costs for clustering in federated learning environments.}
In recent years, federated learning, which allows for collaborative training of machine learning models across distributed entities, has emerged as a central area of research~\cite{wen2023survey}. The Khatri-Rao clustering paradigm can benefit federated learning by reducing communication costs between entities, an important goal in federated learning research~\cite{paragliola2022definition}.

Recently,~\citet{garst2024federated} have introduced an effective implementation of \kmeans in a federated setting (henceforth called \fedkmeans). As \fedkmeans requires communicating centroids between several clients and a server, \fedhadkmeans can potentially reduce the costs of communication by considering a lightweight set of \semicentroidsword instead of regular centroids. To extend \fedkmeans to \kr clustering, it suffices to replace each invocation of \kmeans in \fedkmeans with \hadkmeans.  The resulting algorithm is referred to as \fedhadkmeans. To demonstrate the advantages given by \fedhadkmeans with respect to communication costs, we simulate a federated learning environment with $10$ clients and one server. In this experiment, we consider the benchmark \FEMNIST dataset~\cite{caldas2018leaf}, and we  measure the inertia achieved by \fedkmeans and \fedhadkmeans as a function of the communication costs from the server to the clients. The results of this experiment are displayed in Figure~\ref{fig:federated} and indicate that \fedhadkmeans consistently reduces inertia relative to \fedkmeans at parity communication cost. For the smallest communication cost, the inertia incurred by \fedkmeans is about five times larger than that incurred by \fedhadkmeans. This suggests that \fedhadkmeans can communicate clustering results of a given quality much more efficiently than \fedkmeans.

\section{Conclusion}\label{sec:conclusions}

We have introduced the Khatri-Rao clustering paradigm which extends centroid-based clustering to find more succinct data summaries without significantly compromising their accuracy.

First, we have applied the paradigm in the context of the popular \kmeansprob clustering by introducing the \hadkmeans algorithm. Extensive experiments and case studies show that \hadkmeans can find more succinct but equally accurate data summaries than the long-standing \kmeans algorithm. 

We have also introduced the Khatri-Rao deep clustering framework which hinges on \hadkmeans but fruitfully leverages representation learning to address its lack of flexibility. 
Khatri-Rao deep clustering often performs on par with standard deep clustering, while using considerably fewer parameters.

There are many possibilities for future work. For example, it would be interesting to study potential extensions of the Khatri-Rao clustering paradigm, e.g., by considering a wider range of baseline clustering approaches. However, the most pressing challenge is to understand the underlying structure of the clusters before applying Khatri–Rao clustering algorithms. In particular, how can data exhibiting additive or multiplicative Khatri–Rao structures be effectively characterized?  

\begin{acks}
Aristides Gionis is supported by the ERC Advanced Grant REBOUND (834862), the Swedish Research Council project ExCLUS (2024-05603), and the Wallenberg AI, Autonomous Systems and
Software Program (WASP) funded by the Knut and Alice Wallenberg Foundation.
Collin Leiber is supported by the Research Council of Finland (decision 368654) and Heikki Mannila by the Technology Industries of Finland Centennial Foundation.
\end{acks}


\section*{Artifacts}
Our Python implementations of Khatri-Rao \kmeans, \krdkm and \kridec are available in an online repository~\footnote{\url{https://github.com/maciap/KhatriRaoClustering}}, along with code to load datasets and reproduce the experiments. The repository also contains the Appendix~\footnote{\url{https://github.com/maciap/KhatriRaoClustering/blob/main/Appendix.pdf}}.

\bibliographystyle{ACM-Reference-Format}
\bibliography{ref_article}

\clearpage
\appendix

\section*{APPENDIX}

\iffalse 
\begin{itemize}
    \item \textbf{Appendix \ref{app:dataset_details}:}  \nameref{app:dataset_details} 
     \dotfill Page \pageref{app:dataset_details}

    \smallskip
    
    \item \textbf{Appendix \ref{app:implementation_details}:} \nameref{app:implementation_details}
    \dotfill Page \pageref{app:implementation_details}

    \smallskip

    \item \textbf{Appendix \ref{app:proofs}:} \nameref{app:proofs}
    \dotfill Page \pageref{app:proofs}
\end{itemize}
\fi 

\section{Dataset Details}\label{app:dataset_details}
As stated in Section~\ref{sec:experimental_settings}, in our experiments, we consider $13$ synthetic and real-world datasets commonly used as benchmarks for clustering tasks. The datasets we consider reflect diverse data distributions, clustering structures and application domains. 
Next, we describe each dataset in detail.

\begin{itemize}
    \item The \mnist dataset~\cite{lecun1998gradient}, available via \textsc{PyTorch}~\cite{paszke2017automatic} or \textsc{ClustPy}~\cite{leiber2023benchmarking}, is a collection of $28 \times 28$ (vectorized) images of handwritten grayscale digits. We draw a stratified subsample of 
    $25000$ images, ensuring that the original class proportions across the ten digit clusters are preserved. We rescale all pixel values by dividing by the maximum.

    \item  The \double dataset is derived from \mnist by horizontally concatenating pairs of digit images. Each sample is a $28 \times 56$ (vectorized) grayscale image obtained by placing one $28 \times 28$ digit in the left position and another in the right position, and the label encodes the ordered pair of digits, yielding $100$ clusters in total. For our experiments, we generate $10000$ such composite images using the procedure described above, resulting in an approximately uniform distribution over all digit pairs. Unlike the \mnist dataset, the \double dataset by construction admits a natural clustering with a \kr structure. We rescale all pixel values by dividing by the maximum.

    \item The \har dataset, available via the \textsc{UCI repository}~\cite{kelly2023uci} or \textsc{ClustPy}~\cite{leiber2023benchmarking}, consists of sensor readings collected from smartphone accelerometers and gyroscopes during human activity monitoring. There are $10299$ data points. Each data point is represented as a multivariate feature vector of dimension $561$ derived from raw time-series measurements, and the dataset includes $6$ activity clusters such as walking, standing and sitting.  We standardize each feature by subtracting its mean and dividing by its standard deviation.

    \item The \olivetti dataset, available via \textsc{ClustPy}~\cite{leiber2023benchmarking} or \textsc{scikit-learn}~\cite{scikit-learn} 
contains (vectorized) grayscale facial images of $40$ individuals, with $10$ images per subject captured under varying lighting conditions, facial expressions, and poses. Each image has resolution $64 \times 64$ pixels. We standardize each feature by subtracting its mean and dividing by its standard deviation.

    \item The \cmu dataset, available via the \textsc{UCI repository}~\cite{kelly2023uci} or \textsc{ClustPy}~\cite{leiber2023benchmarking}, contains $624$ grayscale facial images of $20$ persons varying their pose (up, straight, left and right) and expression (neutral, happy, sad, angry), and shown with and without sunglasses. The original images with a resolution of $30 \times 32$ pixels are vectorized and each feature is standardized by subtracting its mean and dividing by its standard deviation.

    \item  The \symbols dataset, available via \textsc{ClustPy}\footnote{The dataset is also available at: \url{https://www.timeseriesclassification.com/index.php}}~\cite{leiber2023benchmarking}, consists of vectorized handwritten symbols. Each sample corresponds to a time series obtained from the drawing trajectory of a symbol. The symbols are drawn by $13$ different individuals. The dataset contains a total of $1020$ data points, each with $398$ measurements. The data points are organized into $6$ natural clusters. We standardize each feature by subtracting its mean and dividing by its standard deviation.

    \item The \stickfigures dataset~\cite{gunnemann2014smvc}, available via \textsc{ClustPy}~\cite{leiber2023benchmarking},  consists of $900$ synthetic (vectorized) silhouette images of human stick figures generated under varying poses. Each image has a resolution of $20 \times 20$ pixels.  Examples of such images are given in Figure~\ref{fig:stickfigures_example}. Each image is provided at a fixed resolution and captures a simplified body configuration defined by joint positions and limb orientations, resulting in $9$ distinct pose-based clusters. We rescale all pixel values by dividing by the maximum.

    \item The \optdigits dataset, available via the \textsc{UCI repository}~\cite{kelly2023uci} or \textsc{ClustPy}~\cite{leiber2023benchmarking}, contains $5620$ (vectorized) handwritten digit images represented as $8 \times 8$ grayscale pixel grids. We standardize each feature by subtracting its mean and dividing by its standard deviation.

    \item The \classification dataset, sourced by \textsc{scikit-learn}~\cite{scikit-learn}, is a synthetic dataset. While it was originally designed for benchmarking classifiers, it is also useful for evaluating clustering algorithms which simply discard class labels until the evaluation of the clustering results. Each class for classification corresponds to a cluster. Except in the experiments where we vary the number of data points, features, or clusters, the dataset consists of $5000$ samples organized into $100$ clusters. The data are generated with $10$ informative features, with no redundant or repeated features, ensuring that all dimensions contribute meaningfully to class separability. We standardize each feature by subtracting its mean and dividing by its standard deviation.
    
    \item The \chameleon dataset, available via the \textsc{clustbench} Benchmark Suite~\cite{gagolewski2022framework}, consists of $10000$ two-dimensional point clouds exhibiting complex, nonconvex cluster shapes with varying densities~\cite{DBLP:journals/computer/KarypisHK99}. The dataset corresponds to the configuration shown at the bottom of Figure~\ref{fig:example_centroids}, augmented with a substantial proportion of uniformly distributed data points that do not belong to any natural cluster.

    \item The \soybean dataset, available via the \textsc{UCI repository}~\cite{kelly2023uci} or \textsc{ClustPy}~\cite{leiber2023benchmarking}, is a  dataset of categorical features. It consists of $562$ plant samples belonging to one of $15$ classes. For each plant, there are $35$ observed categorical attributes describing the plant. After encoding categorical attributes numerically, we standardize each feature by subtracting its mean and dividing by its standard deviation.

    \item The \blobs dataset, sourced by \textsc{scikit-learn}~\cite{scikit-learn}, is a synthetic dataset specifically designed for controlled evaluation of clustering algorithms that consists of data points grouped around isotropic Gaussian clusters. Each cluster has standard deviation $1$. 
    The dataset consists of $5000$ $2$-dimensional data points, arranged in $100$ clusters, except in the experiments where we vary the number of data points, features and clusters. We standardize each feature by subtracting its mean and dividing by its standard deviation.

    \item The \rfifteen dataset, available via the \textsc{clustbench} benchmark suite~\cite{gagolewski2022framework}, consists of $600$ two-dimensional data points forming $15$ Gaussian clusters. The cluster centroids in \rfifteen are not arranged on a regular grid but exhibit non-uniform inter-cluster distances. We standardize each feature by subtracting its mean and dividing by its standard deviation.
\end{itemize}

\section{Implementation Details}\label{app:implementation_details}

In this section, we discuss implementation details for all the algorithms considered in the experiments presented in Section~\ref{sec:experiments}.

\spara{Implementation details of the naïve approach to \kr clustering.}
For the naïve \kr clustering baseline used in our experiments, we first run the \textsc{scikit-learn}~\cite{scikit-learn} implementation of standard \kmeans to extract the desired number of cluster centroids. For instance, if the goal is to extract two sets of $\card_1$ and $\card_2$ \semicentroidsword, \kmeans retrieves $\card_1 \card_2$ clusters. Then, using a coordinate-descent procedure implemented in Python with closed-form updates, we decompose a set of centroids into two smaller sets of \semicentroidsword whose Khatri–Rao product approximates the original centroids. The coordinate-descent procedure alternates between updating the \semicentroidsword of the first  and  second set, holding the other set fixed. At each step, we use the closed-form updates illustrated in Equation~\eqref{eq:update_naive} in Section~\ref{sec:naive}.  The procedure stops either when a maximum number of iterations is reached ($5000$ by default) or when the total sum of squared differences between the initial centroids and the Khatri-Rao aggregation of the \semicentroidsword becomes smaller than a user-specified threshold ($10^{-4}$ by default). Upon termination of the described coordinate-descent procedure, we have the output sets of \semicentroidsword. The corresponding centroids are readily obtained by aggregating \semicentroidsword. To conclude, we assign each data point to the closest centroid. The implementation is available online in our code repository~\footnote{\url{https://github.com/maciap/KhatriRaoClustering/blob/main/scripts/KRkmeansExperimentsLib.py}}. 

\spara{Implementation details of standard \kmeans.}
We rely on the well-established \textsc{scikit-learn} implementation of standard \kmeans. The two \kmeans baselines (namely \kmeans with $\card_1 + \card_2$ centroids and $\card_1 \card_2$ centroids) are obtained by simply specifying the number of centroids $\card_1 + \card_2$ and $\card_1 \card_2$ as input. 
In the scalability experiments, instead, to ensure a fair comparison, we use an implementation of \kmeans which mirrors the implementation of \hadkmeans described next.  

\spara{Implementation details of \hadkmeans.}
Our experiments rely on a simple Python implementation of \hadkmeans that is purely built on \textsc{NumPy}~\footnote{\url{https://numpy.org}}, taking advantage of vectorized operations for efficiency. Such implementation is available online~\footnote{\url{https://github.com/maciap/KhatriRaoClustering/tree/main/KathriRaokMeans}}.

For initialization, by default we sample random data points as the initial \semicentroidsword (as in Algorithm~\ref{alg:hadkmeans}). Alternatively, we can adopt the strategy inspired by $\kmeanspp$ described in Section~\ref{sec:algorithms_shallow}, which selects representative data points based on distance criteria. Our implementation of this strategy either deterministically chooses the data point farthest from the previously selected centroids, or samples data points with probability proportional to their distance from those centroids. After initialization, we  iteratively compute assignments, \semicentroidsword and corresponding centroids. Thanks to the closed-form updates introduced in Section~\ref{sec:algorithms_shallow}, the updates of \semicentroidsword
(and centroids) are implemented in a fully-vectorized manner. At each iteration we monitor convergence by tracking the movement of all reconstructed centroids; the algorithm terminates once this movement falls below a user-specified threshold or a maximum number of iterations is reached.

During the execution of the algorithm, in case empty clusters arise, they are handled by re\-initializing the corresponding \semicentroidword to a random data point.

Our implementation is deliberately simple and easy to follow. 
In the future, more optimized or parallelized implementations could be developed for larger-scale settings.

\hadkmeans admits a time-efficient and a memory-efficient implementation. 
Algorithm~\ref{alg:hadkmeans} presents the memory-efficient implementation 
that avoids storing the full set of centroids by computing them on the fly from the stored set of \semicentroidsword.
This approach can reduce memory requirements, particularly when the number of clusters is large since memory requirements only grow additively with the total number of \semicentroidsword instead of multiplicatively. For example, $\card_1 \card_2$ clusters demand storing only $\card_1 + \card_2$ \semicentroidsword instead of $\card_1 \card_2$ centroids. However, computing centroids  on the fly incurs a runtime overhead. A time-efficient implementation is obtained by computing centroids once and storing them.

\spara{Implementation details of standard deep clustering algorithms.} For \dkm and \idec, we rely on the off-the-shelf \textsc{PyTorch}-based implementations provided by  the \textsc{ClustPy} library~\footnote{https://github.com/collinleiber/ClustPy}. To find initial centroids, \dkm and \idec use the \textsc{scikit-learn} implementation of standard \kmeans. 

\spara{Implementation details of \kr deep clustering algorithms.}
For the implementation of \kr deep clustering algorithms, 
we build the implementation of \krdkm and \kridec on top of the \textsc{ClustPy} implementations of the corresponding standard deep clustering algorithms.

Extending the \textsc{ClustPy} implementation of a standard deep clustering algorithm to the \kr clustering paradigm is straightforward. During initialization, we rely on our implementation of \hadkmeans to obtain initial \semicentroidsword. We then re\-parameterize the centroids to satisfy the \kr structure and adjust the autoencoder parameters to conform to the Hadamard-decomposition re\-parameterization. Our implementation of \kr deep clustering algorithms is available online~\footnote{\url{https://github.com/maciap/KhatriRaoClustering/tree/main/KhatriRaoDeepClustering}}.

\section{Proofs}\label{app:proofs}

In this section, we collect the proofs of the propositions stated in the paper.

\iffalse 
\spara{Proof of Equation~\eqref{eq:update_naive}.}
\begin{proof}
We consider the case $\nsets=2$ and $\aggregator=\odot$. Let $\centroidsset=\{\centroid_{i,j}\}$ be the centroids obtained in the first phase, and consider the objective
\[
\sum_{i=1}^{\card_1}\sum_{j=1}^{\card_2}
\|\centroid_{i,j}-\protocentroid^i_1\odot\protocentroid^j_2\|^2 .
\]
Fix the second set $\{\protocentroid^j_2\}_{j=1}^{\card_2}$. Only the $\card_2$ centroids indexed by $(i,j)$ depend on $\protocentroid^i_1$, yielding the gradient
\[
2\sum_{j=1}^{\card_2}
\left(
\centroid_{i,j}-\protocentroid^i_1\odot\protocentroid^j_2
\right)\odot\protocentroid^j_2,
\]
Equating the gradient to zero and solving element-wise gives
\[
\protocentroid^i_1
=
\frac{\sum_{j=1}^{\card_2}\centroid_{i,j}\odot\protocentroid^j_2}
{\sum_{j=1}^{\card_2}\protocentroid^j_2\odot\protocentroid^j_2}.
\]
\end{proof}
\fi

\spara{Proof of Proposition~\ref{prop:closed_form_updates}.}
\begin{proof}
We illustrate the proof for the product aggregator. Consider the $j$-th \semicentroidword in the first set of \semicentroidsword. The optimal update for this \semicentroidword satisfies 
\[
\protocentroid^{j}_1  = \argmin_{\protocentroid} \sum_{l=1}^{\card_2} 
       \sum_{\xbf \in \cluster_{j,l}} (\xbf - \protocentroid  \odot \protocentroid_2^l)^2.  
\]
Therefore $\protocentroid^{j}_1$ can be found by computing the gradient of this sum of squared differences with respect to  $\protocentroid$ and equating it to the zero vector $\zerovector$. Doing so, one obtains 
\[
-2 \sum_{l=1}^{\card_2} 
       \sum_{\xbf \in \cluster_{j,l}}  (\xbf - \protocentroid  \odot \protocentroid_2^l) \protocentroid_2^l = \zerovector,
\]
which holds if and only if: 
\[
\protocentroid^{j}_1 = 
       \frac{
       \sum_{l=1}^{\card_2} 
       \sum_{\xbf \in \cluster_{j,l}} \xbf \odot \protocentroid^l_2 } 
        {
        \sum_{l=1}^{\card_2} |\cluster_{j,l}|
       %\sum_{\xbf \in \cluster_{j,l}} 
       {\protocentroid^l_2}  \odot {\protocentroid^l_2},
        }. 
\]
The derivations for the second set of \semicentroidsword are symmetric, and the proof for the sum aggregator is also similar. 
\end{proof}

\spara{Proof of Proposition~\ref{prop:maximizer}.}

\begin{proof}
Assuming $\nsets$ \semicentroidword\ sets of equal size and exact budget usage, each set contains $\card=\frac{\budget}{\nsets}$ \semicentroidsword, so the total number of centroids that can be represented is:
\[
\left(\frac{\budget}{\nsets}\right)^{\nsets}.
\]
Maximizing this function over $\nsets>0$ is equivalent to maximizing its natural logarithm
\[
\log \left(\frac{\budget}{\nsets}\right)^{\nsets} =  \nsets \log\!\left(\frac{\budget}{\nsets}\right)
= \nsets \log \budget- \nsets \log \nsets. 
\]
Differentiating with respect to $\nsets$ yields
\[
\log\!\left(\frac{\budget}{\nsets}\right)-1,
\]
which is positive for $\nsets<\frac{\budget}{e}$ and negative for $\nsets>\frac{\budget}{e}$. Moreover, the second derivative
\[
-\frac{1}{\nsets}<0\qquad(\nsets>0)
\]
shows that the log-objective is strictly concave, and therefore attains a unique maximum at $\nsets=\frac{\budget}{e}$ in the continuous domain.

When restricting to the divisor-only setting,  the optimal value of $\nsets$, $\nsets^{max}$,  is required to be an integer that exactly divides $\budget$, so that $\card^{max}=\frac{\budget}{\nsets^{max}}$ is an integer. Since the objective is strictly increasing for $\nsets<\frac{\budget}{e}$ and strictly decreasing for $\nsets>\frac{\budget}{e}$, among all admissible values the maximum is attained at one of the two values of $\nsets$ that are immediately below or above $\frac{\budget}{e}$ (when each exists). All other admissible divisors are necessarily further away from $\frac{\budget}{e}$, and therefore yield a strictly smaller objective value.
\end{proof}

\spara{Proof of Proposition~\ref{prop:number_of_sets_bound}.}
\begin{proof}
     We present the proof for the product aggregator. However, extension to different aggregator functions is straightforward. 

     First,  to  see why it must be $\nsets^* \geq   \log_{\card_{min}} \nclusters $, notice that, if $\nsets^* < \log_{\card_{min}} \nclusters$ then $\card_{min}^{\nsets^*} <  \nclusters$, and $\card_{min}^{\nsets^*}$ is the maximum number of centroids that can be represented using $\nsets^*$ sets of at least $\card_{min}$ \semicentroidsword.
     As regards the other inequality, to prove that $\nsets^* \leq \lceil \frac{\nclusters}{\card_{min}-1} \rceil$, it is sufficient to construct an example where all sets have a protocentroid with all entries equal to $1$ and the remaining at least $\card_{min}-1$ protocentroids that are equal to centroids. Different sets contain different centroids. 
     In the illustrated scenario, each set of protocentroids represents exactly at least $\card_{min}-1$ centroids.  Thus, $\lceil \frac{\nclusters}{\card_{min}-1} \rceil$ sets of \semicentroidsword can always represent $\nclusters$ centroids. 
\end{proof}

\end{document}